\documentclass{article}

\PassOptionsToPackage{numbers, compress}{natbib}
 \usepackage[preprint]{neurips_2026}

\usepackage[utf8]{inputenc} 
\usepackage[T1]{fontenc}    
\usepackage{hyperref}       
\usepackage{url}            
\usepackage{booktabs}       
\usepackage{amsfonts}       
\usepackage{nicefrac}       
\usepackage{microtype}      
\usepackage{xcolor}         
\usepackage{amsmath}        
\usepackage{enumitem}
\usepackage{mathtools}
\usepackage{amsthm}
\newtheorem{theorem}{Theorem}
\newtheorem{knownresult}{Known Result}

\theoremstyle{definition}

\theoremstyle{remark}

\usepackage{adjustbox}
\usepackage{multirow}
\usepackage{colortbl}
\usepackage{tabularx}

\newcolumntype{L}{>{\raggedright\arraybackslash}X}

\usepackage{amssymb}

\usepackage{subcaption}

\usepackage{graphicx}
\usepackage{tikz}
\usetikzlibrary{arrows.meta,calc,shapes.geometric}

\definecolor{spaceFill}{RGB}{244,224,174}
\definecolor{actionRed}{RGB}{196,78,82}
\definecolor{softGray}{RGB}{150,150,150}
\definecolor{infoPink}{RGB}{238,120,155}

\usepackage{pgfplots}
\pgfplotsset{compat=1.18}

\title{QuoVLA: Quotient Space for Vision-Language-Action Models}

\author{%
  Xuan Wang \\
  Department of Automation\\
  Tsinghua University\\
  \texttt{xwangrs@gmail.com} \\
  \And
  Yinan Wu \\
  Department of Automation\\
  Tsinghua University\\
  \texttt{yinanwu@ieee.org} \\
  \And
  Haoran Duan \\
  Department of Automation\\
  Tsinghua University\\
  \texttt{haoran.duan@ieee.org} \\
  \And
  Jungong Han\thanks{Corresponding author.}\\
  Department of Automation\\
  Tsinghua University\\
  \texttt{jghan@tsinghua.edu.cn} \\
}

\begin{document}

\maketitle

\begin{abstract}
Vision-Language-Action (VLA) models commonly adapt pretrained Vision-Language Models (VLMs) to robot control by mapping visual observations and language instructions to continuous actions. Existing approaches typically take an action-insufficiency view, assuming that pretrained VLM latents either lack directly usable action information or should be shielded from action-learning signals. Against this view, our \textit{Quotient Theory for VLA} shows that pretrained VLM latents are not action-insufficient but action-sufficient: they already contain the information needed for control, yet remain overcomplete by distinguishing prompt-level variations that induce the same optimal action behavior. To operationalize this theory, we propose QuoVLA, a quotient-space framework for VLA that compresses pretrained VLM latents into action-sufficient representations. Specifically, QuoVLA instantiates this principle with a quantization module and a dual-branch design with relative temporal-complexity regularization, preserving action-relevant information while removing prompt-level redundancy. Extensive experiments across multiple benchmarks demonstrate that QuoVLA achieves strong performance, with particularly notable improvements in generalization under visual, linguistic, and environmental distribution shifts. Our code will be made publicly available.
\end{abstract}

\section{Introduction}

Recent advances in Vision-Language Models (VLMs) \cite{radford2021learning,liu2023visual} have made pretrained representations a natural foundation for robot control, motivating Vision-Language-Action (VLA) models that map visual observations and language instructions to robot actions \cite{kim2024openvla,black2024pi_0,intelligence2025pi_05}. Yet representations optimized for multimodal understanding are not inherently suited to control. They must be adapted for action generation \cite{zhang2026vlm4vla,kim2025contrastive} while preserving the knowledge acquired during pretraining \cite{driess2025knowledge,yang2025instructvla}.

Existing methods cast this adaptation as a VLM-Action mismatch, but differ in how they diagnose its source. Some attribute the mismatch to insufficient action-relevant information in pretrained VLM representations, motivating further VLM fine-tuning for control \cite{zhang2026vlm4vla,du2026embodiedmidtrain,zitkovich2023rt,ye2026st4vla}. Others attribute it to the interference of action supervision, motivating designs that isolate the VLM from action-learning signals \cite{driess2025knowledge,intelligence2026pi}. Both views therefore converge on an action-insufficiency diagnosis, assuming that pretrained VLM latents lack sufficient usable information for action generation.

Against this action-insufficiency diagnosis, we develop a \emph{Quotient Theory for VLA} that reassesses pretrained VLM representations through the lens of action. The theory shows that these representations can already be action-sufficient, containing enough information to determine optimal actions, yet still encode variability that is irrelevant to the resulting action and can hinder generalization. As shown in Fig. \ref{fig:action_induced_quotient}, VLA adaptation should therefore pursue an \emph{action-minimal} representation, retaining only the minimal sufficient information for optimal action behavior. The next two paragraphs elaborate this criterion from the perspectives of (1) \textbf{minimality} and (2) \textbf{action}, respectively.

\begin{figure}[t]
\centering

\begin{subfigure}[t]{0.375\linewidth}
\centering
\resizebox{\linewidth}{!}{%
\begin{tikzpicture}[
    x=1cm,
    y=1cm,
    >=Latex,
    every node/.style={font=\scriptsize},
    redarrow/.style={-{Latex[length=1.8mm]}, line width=0.55pt, actionRed, opacity=0.72},
    infocircle/.style={draw=infoPink!80!black, dashed, line width=0.65pt, fill=infoPink!16}
]

\path[use as bounding box] (0.00,0.45) rectangle (4.85,3.60);

\def\RepPathA{
    (0.34,2.08)
    .. controls (0.20,2.40) and (0.54,2.68) .. (0.90,2.58)
    .. controls (1.12,2.52) and (1.22,2.82) .. (1.52,2.70)
    .. controls (1.82,2.58) and (2.02,2.82) .. (2.28,2.55)
    .. controls (2.55,2.28) and (2.42,1.96) .. (2.05,1.82)
    .. controls (1.78,1.72) and (1.58,1.42) .. (1.22,1.52)
    .. controls (0.92,1.60) and (0.68,1.36) .. (0.48,1.62)
    .. controls (0.32,1.82) and (0.46,1.96) .. (0.34,2.08)
    -- cycle
}

\def\ActPathA{
    (3.36,2.10)
    .. controls (3.34,2.48) and (3.66,2.76) .. (4.04,2.72)
    .. controls (4.42,2.70) and (4.74,2.43) .. (4.72,2.06)
    .. controls (4.70,1.70) and (4.38,1.42) .. (4.02,1.42)
    .. controls (3.68,1.42) and (3.40,1.64) .. (3.36,1.94)
    .. controls (3.35,2.00) and (3.37,2.05) .. (3.36,2.10)
    -- cycle
}

\node[font=\scriptsize\bfseries] at (1.35,3.24)
    {VLM Space $\mathcal{Z}$};

\path[draw=black, fill=spaceFill, line width=1.05pt, line join=round]
    \RepPathA;

\foreach \p in {
    (0.70,2.42),
    (0.78,1.70),
    (1.42,2.58),
    (1.92,2.30),
    (1.72,2.48),
    (0.96,2.15),
    (1.56,1.82)
}{
    \fill[black] \p circle (0.95pt);
}

\coordinate (z1a) at (0.9,2.02);
\coordinate (z2a) at (1.26,2.32);
\coordinate (z3a) at (1.65,2.02);

\draw[infocircle]
    (1.30,2.12) ellipse [x radius=0.58, y radius=0.36];

\filldraw[fill=actionRed, draw=black, line width=0.30pt]
    (z1a) circle (2.15pt);
\filldraw[fill=actionRed, draw=black, line width=0.30pt]
    (z2a) circle (2.15pt);
\filldraw[fill=actionRed, draw=black, line width=0.30pt]
    (z3a) circle (2.15pt);

\node[font=\scriptsize\bfseries] at ($(z1a)+(-0.16,-0.18)$) {$z_1$};
\node[font=\scriptsize\bfseries] at ($(z2a)+(0.00,0.24)$) {$z_2$};
\node[font=\scriptsize\bfseries] at ($(z3a)+(0.18,-0.18)$) {$z_3$};

\node[
    font=\scriptsize,
    align=center,
    text=infoPink!65!black
] at (1.30,0.86)
    {action-irrelevant\\information};

\draw[-{Latex[length=1.35mm]}, infoPink!65!black, line width=0.45pt]
    (1.30,1.08) to[out=90,in=-90] (1.30,1.72);

\node[font=\scriptsize\bfseries] at (4.05,3.24)
    {Action Space $\mathcal{A}$};

\path[draw=black, fill=spaceFill!45, line width=1.05pt, line join=round]
    \ActPathA;

\coordinate (a1a) at (3.98,2.04);
\coordinate (a2a) at (4.05,2.12);
\coordinate (a3a) at (4.12,2.04);

\draw[infocircle]
    (4.05,2.07) ellipse [x radius=0.21, y radius=0.15];

\filldraw[fill=actionRed, draw=black, line width=0.28pt]
    (a1a) circle (1.20pt);
\filldraw[fill=actionRed, draw=black, line width=0.28pt]
    (a2a) circle (1.20pt);
\filldraw[fill=actionRed, draw=black, line width=0.28pt]
    (a3a) circle (1.20pt);

\node[font=\tiny\bfseries, anchor=east]
    at ($(a1a)+(-0.05,-0.12)$) {$a_1$};
\node[font=\tiny\bfseries, anchor=south]
    at ($(a2a)+(0.00,0.09)$) {$a_2$};
\node[font=\tiny\bfseries, anchor=west]
    at ($(a3a)+(0.05,-0.08)$) {$a_3$};

\node[
    font=\scriptsize,
    align=center,
    text=infoPink!65!black
] at (4.05,0.86)
    {action-irrelevant\\information};

\draw[-{Latex[length=1.35mm]}, infoPink!65!black, line width=0.45pt]
    (4.05,1.08) to[out=90,in=-90] (4.05,1.90);

\draw[redarrow]
    (z1a) to[out=4,in=178] ($(a1a)+(-0.08,0.00)$);

\draw[redarrow]
    (z2a) to[out=0,in=180] ($(a1a)+(-0.08,0.00)$);

\draw[redarrow]
    (z3a) to[out=-4,in=182] ($(a1a)+(-0.08,0.00)$);


\node[font=\scriptsize] at (2.68,0.58)
    {$\pi^*(z_i)=a_i$};

\end{tikzpicture}%
}
\caption{Direct VLM-to-Action mapping.}
\label{fig:direct_vlm_action}
\end{subfigure}
\hfill
\begin{subfigure}[t]{0.585\linewidth}
\centering
\resizebox{\linewidth}{!}{%
\begin{tikzpicture}[
    x=1cm,
    y=1cm,
    >=Latex,
    every node/.style={font=\scriptsize},
    redarrow/.style={-{Latex[length=1.8mm]}, line width=0.55pt, actionRed, opacity=0.72},
    cellline/.style={softGray!65, line width=0.42pt, opacity=0.55, line cap=round},
    infocircle/.style={draw=infoPink!80!black, dashed, line width=0.65pt, fill=infoPink!16}
]

\path[use as bounding box] (0.00,0.45) rectangle (7.55,3.60);

\def\RepPathB{
    (0.34,2.08)
    .. controls (0.20,2.40) and (0.54,2.68) .. (0.90,2.58)
    .. controls (1.12,2.52) and (1.22,2.82) .. (1.52,2.70)
    .. controls (1.82,2.58) and (2.02,2.82) .. (2.28,2.55)
    .. controls (2.55,2.28) and (2.42,1.96) .. (2.05,1.82)
    .. controls (1.78,1.72) and (1.58,1.42) .. (1.22,1.52)
    .. controls (0.92,1.60) and (0.68,1.36) .. (0.48,1.62)
    .. controls (0.32,1.82) and (0.46,1.96) .. (0.34,2.08)
    -- cycle
}

\def\QuotPathB{
    (3.02,2.08)
    .. controls (2.92,2.42) and (3.26,2.70) .. (3.62,2.62)
    .. controls (3.88,2.56) and (4.04,2.82) .. (4.36,2.70)
    .. controls (4.68,2.58) and (4.92,2.30) .. (4.84,2.00)
    .. controls (4.76,1.66) and (4.42,1.58) .. (4.12,1.46)
    .. controls (3.78,1.32) and (3.36,1.44) .. (3.18,1.72)
    .. controls (3.06,1.88) and (3.08,2.00) .. (3.02,2.08)
    -- cycle
}

\def\ActPathB{
    (5.86,2.10)
    .. controls (5.84,2.48) and (6.16,2.76) .. (6.54,2.72)
    .. controls (6.92,2.70) and (7.24,2.43) .. (7.22,2.06)
    .. controls (7.20,1.70) and (6.88,1.42) .. (6.52,1.42)
    .. controls (6.18,1.42) and (5.90,1.64) .. (5.86,1.94)
    .. controls (5.85,2.00) and (5.87,2.05) .. (5.86,2.10)
    -- cycle
}

\node[font=\scriptsize\bfseries] at (1.35,3.24)
    {VLM Space $\mathcal{Z}$};

\path[draw=black, fill=spaceFill, line width=1.05pt, line join=round]
    \RepPathB;

\foreach \p in {
    (0.70,2.42),
    (0.78,1.70),
    (1.42,2.58),
    (1.92,2.30),
    (1.72,2.48),
    (0.96,2.15),
    (1.56,1.82)
}{
    \fill[black] \p circle (0.95pt);
}

\coordinate (z1b) at (0.95,2.02);
\coordinate (z2b) at (1.30,2.32);
\coordinate (z3b) at (1.65,2.02);

\draw[infocircle]
    (1.30,2.12) ellipse [x radius=0.58, y radius=0.36];

\filldraw[fill=actionRed, draw=black, line width=0.30pt]
    (z1b) circle (2.15pt);
\filldraw[fill=actionRed, draw=black, line width=0.30pt]
    (z2b) circle (2.15pt);
\filldraw[fill=actionRed, draw=black, line width=0.30pt]
    (z3b) circle (2.15pt);

\node[font=\scriptsize\bfseries] at ($(z1b)+(-0.16,-0.18)$) {$z_1$};
\node[font=\scriptsize\bfseries] at ($(z2b)+(0.00,0.24)$) {$z_2$};
\node[font=\scriptsize\bfseries] at ($(z3b)+(0.18,-0.18)$) {$z_3$};

\node[
    font=\scriptsize,
    align=center,
    text=infoPink!65!black
] at (1.3,0.86)
    {action-irrelevant\\information};

\draw[-{Latex[length=1.35mm]}, infoPink!65!black, line width=0.45pt]
    (1.3,1.08) to[out=90,in=-90] (1.30,1.72);

\node[font=\scriptsize\bfseries] at (3.95,3.24)
    {Quotient Space $\mathcal{Z}/\!\mathcal{A}$};

\path[fill=spaceFill!70]
    \QuotPathB;

\begin{scope}
    \clip \QuotPathB;

    \draw[cellline]
        (3.20,1.50)
        .. controls (3.44,1.80) and (3.50,2.30) ..
        (3.56,2.62);

    \draw[cellline]
        (3.92,1.40)
        .. controls (3.74,1.78) and (4.04,2.22) ..
        (3.98,2.68);

    \draw[cellline]
        (4.48,1.62)
        .. controls (4.24,1.90) and (4.42,2.30) ..
        (4.62,2.52);

    \draw[cellline]
        (3.10,1.86)
        .. controls (3.62,2.02) and (4.18,1.88) ..
        (4.76,2.06);

    \draw[cellline]
        (3.18,2.38)
        .. controls (3.72,2.24) and (4.20,2.46) ..
        (4.72,2.30);

\end{scope}

\path[draw=black, fill=none, line width=1.05pt, line join=round]
    \QuotPathB;

\coordinate (cstarb) at (3.95,2.10);

\filldraw[fill=actionRed, draw=black, line width=0.36pt]
    (cstarb) circle (2.55pt);

\node[font=\scriptsize\bfseries, anchor=west]
    at ($(cstarb)+(0.22,0.20)$) {$[z]$};

\draw[redarrow]
    (z1b) to[out=8,in=180] ($(cstarb)+(-0.16,0.00)$);

\draw[redarrow]
    (z2b) to[out=0,in=180] ($(cstarb)+(-0.16,0.00)$);

\draw[redarrow]
    (z3b) to[out=-8,in=180] ($(cstarb)+(-0.16,0.00)$);

\node[font=\scriptsize\bfseries] at (6.55,3.24)
    {Action Space $\mathcal{A}$};

\path[draw=black, fill=spaceFill!45, line width=1.05pt, line join=round]
    \ActPathB;

\coordinate (astarb) at (6.55,2.10);

\filldraw[fill=actionRed, draw=black, line width=0.38pt]
    (astarb) circle (2.75pt);

\node[font=\scriptsize\bfseries, anchor=west]
    at ($(astarb)+(0.23,0.00)$) {$a^*$};

\draw[redarrow]
    ($(cstarb)+(0.18,0.00)$) -- ($(astarb)+(-0.18,0.00)$);

\node[font=\scriptsize] at (4.55,0.58)
    {$q(z_i)=[z],\quad \bar{\pi}^{*}([z])=a^*$};

\end{tikzpicture}%
}
\caption{Action-induced quotient mapping.}
\label{fig:action_induced_quotient_b}
\end{subfigure}

\caption{
\textbf{Direct VLM-to-Action mapping \emph{versus} action-induced quotient mapping.}
(a) Transformer injectivity suggests that pretrained VLM latents preserve action-irrelevant prompt details\cite{nikolaou2026language}. Without quotienting, these details may propagate to the action space, causing spurious action variations and poorer generalization.
(b) With the action-induced quotient, action-irrelevant variability is collapsed before decoding the shared quotient code \([z]\) into the optimal action \(a^*\).
}
\label{fig:action_induced_quotient}
\end{figure}
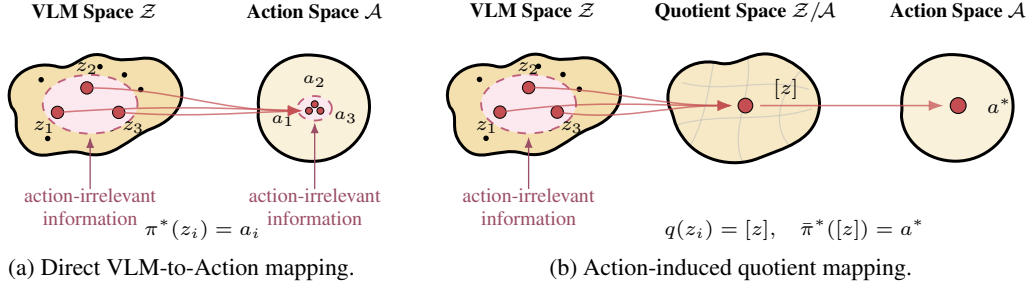

Under the \textbf{minimality} criterion, action-minimal adaptation requires an explicit Quotient Space that collapses prompt-distinct yet action-equivalent latents. Such a space is motivated by the injectivity of Transformers, which suggests that pretrained latents may losslessly encode prompts and thus retain action-irrelevant variability \cite{nikolaou2026language}. To explicitly discard such redundancy, we introduce a quantization module that maps continuous VLM latents to shared discrete codes. During training, we introduce an adaptive variant of the Straight-Through Estimator \cite{bengio2013estimating,courbariaux2016binarized} to modulate gradients propagated back to the VLM, thereby reducing interference with pretrained multimodal knowledge.

Under the \textbf{action} criterion, this Quotient Space must preserve the action behavior induced by the original latent while removing only action-redundant variability. To this end, we employ a dual-branch action-generation design: a stop-gradient unquantized branch predicts actions from the original VLM latent as a reference, while a quantized branch predicts actions from the compressed latent. We further impose a relative temporal-complexity constraint to keep the quantized branch no more complex than the reference, thereby preserving action sufficiency while eliminating action-redundant information.

Our contributions are threefold:
\begin{itemize}
    \item We propose, for the first time, a \textit{Quotient Theory for VLA} that challenges the action-insufficiency diagnosis and establishes action-minimality as a representation principle for generalizable VLA adaptation.
    \item We propose an active quantization mechanism that explicitly removes redundant information from VLM representations and controls straight-through gradients to mitigate interference from action supervision.
    \item We introduce a dual-branch action-generation design that uses a stop-gradient unquantized branch as a reference to guide the quantized branch toward preserving action-relevant information while eliminating action-redundant information.
\end{itemize}

\section{Related Work}

\paragraph{Fine-tuning pretrained VLMs for action generation.}

A major line of VLA methods adapts pretrained
VLMs to robot control by further tuning them on action-labeled robot data. These methods are motivated by the observation that representations learned from web-scale vision-language pretraining are not directly optimized for continuous control. Accordingly, existing VLA methods differ less in whether they use action-labeled robot data than in how they expose pretrained visual-language representations to action supervision. 
RT-2 and OpenVLA preserve the autoregressive language-model interface by discretizing continuous robot actions into text or action tokens and fine-tuning the VLM with next-token prediction \cite{zitkovich2023rt,kim2024openvla}. A more recent direction seeks architectural unification by representing vision, language, and action within a single tokenized autoregressive framework, enabling the same sequence model to jointly capture multimodal understanding, temporal dynamics, and action generation \cite{bu2025univla}. Other methods attach explicit policy or action decoders to pretrained VLM features, ranging from RoboFlamingo-style \cite{wang2025roboflamingo} action heads to recent continuous diffusion or flow-matching action experts used in $\pi_0$ and $\pi_{0.5}$ for chunked action generation \cite{black2024pi_0,intelligence2025pi_05}. However, these
approaches largely follow an action-insufficiency view: they treat pretrained
VLM latents as lacking sufficient action-relevant information, and therefore
seek to make the VLM representation more action-aware through further training.

\paragraph{Preventing action learning from interfering with pretrained VLMs.} A complementary line of work starts from an interference-mitigation perspective: instead of fully exposing the pretrained VLM to robot-specific action losses, it aims to preserve its visual-language knowledge while confining action learning to separate or lightweight modules. From this perspective, existing VLAs can be viewed as a spectrum according to how strongly they separate action learning from the pretrained backbone. At the decoding end, compact VLAs such as TinyVLA retain an efficient multimodal backbone and delegate continuous control to a dedicated diffusion policy decoder, concentrating action adaptation on the control-decoding side \citep{wen2025tinyvla}. Moving one step closer to the backbone, adapter-based methods such as VLA-Adapter introduce lightweight bridge or policy modules between visual-language features and the action space, reducing the need to extensively update the pretrained model with robot data \citep{wang2026vla}. At the strongest end of this spectrum, knowledge-insulated VLAs such as Knowledge Insulating VLA and $\pi_{0.7}$ decouple representation learning from continuous control by supervising the VLM with discrete action tokens while training a separate diffusion or flow-matching action expert whose gradients are blocked from the VLM \citep{driess2025knowledge,intelligence2026pi}. Although these designs effectively reduce action-induced interference, they primarily do so by isolating action gradients, rather than by explicitly collapsing action-equivalent VLM latents.

In contrast to both fine-tuning-based and insulation-based VLA adaptation, QuoVLA does not seek to add action information to the VLM or merely block action-induced interference, but instead removes action-redundant variability by learning an action-preserving quotient of pretrained VLM latents.

\section{Quotient Theory for VLA}
\label{sec:theory}

\subsection{Preliminary}

\paragraph{VLA formulation.}
Let \(P_X\) denote the data distribution over visual-language inputs. We write an input as
\(x=(o_{1:K},\ell)\in\mathcal X_{\mathrm{VL}}\), where
\(o_{1:K}\in\mathcal O^K\) is the observation history and
\(\ell\in\mathcal L\) is the language instruction. A VLA policy predicts an action trajectory
\(a_{1:T}\in\mathcal A_{\mathrm{seq}}:=\mathcal A^T\) through a VLM representation map
\(h_\theta\) and an action head \(G_\phi^{\mathrm{act}}\):
\begin{equation}
h=h_\theta(x)\in\mathcal H,
\qquad
\pi_{\theta,\phi}(\cdot\mid x)
:=
G_\phi^{\mathrm{act}}(\cdot\mid h)
\in \Delta(\mathcal A_{\mathrm{seq}}).
\end{equation}
Here \(h_\theta:\mathcal X_{\mathrm{VL}}\to\mathcal H\) maps the input to a VLM latent representation,
\(G_\phi^{\mathrm{act}}:\mathcal H\to\Delta(\mathcal A_{\mathrm{seq}})\) maps a latent representation to a distribution over action trajectories,
and \(\pi_{\theta,\phi}\) denotes the resulting VLA policy parameterized by the VLM parameters \(\theta\) and the action-head parameters \(\phi\). 


\paragraph{Injectivity of VLA representations.}

We invoke the finite-context injectivity result for Transformer representations~\citep{nikolaou2026language}.
In our VLA setting, the multimodal tokenization
\(\tau:\mathcal X_{\mathrm{VL}}\to\mathcal V_{\mathrm{VL}}^{\le K_{\mathrm{VL}}}\)
is injective on the considered finite input set. Applying this result to the composed VLA map
\(G_\phi^{\mathrm{act}}\circ h_\theta\), the resulting policy map
\(\pi_{\theta,\phi}(\cdot\mid x)\) is almost surely injective:
\begin{equation}
x\neq x'
\quad\Longrightarrow\quad
\pi_{\theta,\phi}(\cdot\mid x)
\neq
\pi_{\theta,\phi}(\cdot\mid x'),
\qquad
x,x'\in\mathcal X_{\mathrm{VL}}
\end{equation}

Here \(\pi_{\theta,\phi}(\cdot\mid x)\) denotes the action-trajectory distribution produced by the VLA.  Consequently, the injective VLA map can propagate action-irrelevant visual or linguistic variations into the action space. 
This may assign different action-output distributions to inputs that should share the same optimal control behavior, creating spurious action distinctions and weakening generalization, which motivates the action-induced quotient introduced below.

\subsection{Minimal VLM-Action Quotient}
\label{sec:minimal_vlm_action_quotient}

The injectivity result above implies that a VLA may preserve input differences that are not needed for control. What matters for action generation is not whether two visual-language inputs are different, but whether they require different optimal action behavior. To formalize this, let \(P_{X,A}\) be the joint distribution of visual-language inputs \(X\in\mathcal X_{\mathrm{VL}}\) and expert action trajectories
\(A_{1:T}\in\mathcal A_{\mathrm{seq}}\), with marginal \(P_X\). Given an action loss \(\ell_{\mathrm{act}}\), define the Bayes-optimal action-trajectory law by:
\begin{equation}
p_X^\star(\cdot\mid x)
\in
\arg\min_{\mu\in\Delta(\mathcal A_{\mathrm{seq}})}
\mathbb E_{P_{X,A}}
\bigl[
\ell_{\mathrm{act}}(\mu,A_{1:T})
\mid X=x
\bigr].
\end{equation}

This law represents the optimal action-trajectory distribution for input \(x\). We now compare VLM latents by the optimal action behavior they induce. The Bayes-optimal law depends on \(x\) through the VLM latent
\(h_\theta(x)\), namely $p_X^\star(\cdot\mid x)=p_H^\star(\cdot\mid h_\theta(x))$.  Let \(H=h_\theta(X)\) be the VLM latent induced by \(X\sim P_X\), and let
\(P_H=(h_\theta)_\#P_X\) denote its distribution. On the support
\(\mathcal H_{\mathrm{sup}}:=\operatorname{supp}(P_H)\), define the action
equivalence relation:
\begin{equation}
h\sim_{\mathcal A}h'
\quad\Longleftrightarrow\quad
p_H^\star(\cdot\mid h)
=
p_H^\star(\cdot\mid h').
\end{equation}

Thus, two VLM latents are equivalent if no Bayes-optimal action predictor needs
to distinguish them.

\begin{theorem}[Minimal VLM-Action quotient (Proof detailed in
Appendix~\ref{app:proof_minimal_vlm_action_quotient}.)]
\label{thm:minimal_vlm_action_quotient}
The quotient map
\begin{equation}
q_{\mathcal A}^\star(h):=[h]_{\sim_{\mathcal A}},
\qquad
\mathcal H_{\mathcal A}^\star
:=
\mathcal H_{\mathrm{sup}}/\sim_{\mathcal A}
\end{equation}
is exact action-sufficient: there exists
\(\bar p_{\mathcal A}^\star:
\mathcal H_{\mathcal A}^\star\to\Delta(\mathcal A_{\mathrm{seq}})\)
such that \(
p_X^\star(\cdot\mid x)
=
\bar p_{\mathcal A}^\star
\bigl(q_{\mathcal A}^\star(h_\theta(x))\bigr).
\) Moreover, \(q_{\mathcal A}^\star\) is the coarsest representation with this
property: any exact action-sufficient representation can only further split,
but not merge, its equivalence classes.
\end{theorem}

Theorem~\ref{thm:minimal_vlm_action_quotient} identifies the ideal VLA representation as the action quotient \(q_{\mathcal A}^\star(h_\theta(x))\), rather than the original VLM latent \(h_\theta(x)\). This quotient preserves exactly the latent distinctions that change the Bayes-optimal action-trajectory
law and removes those irrelevant to action generation. Since
\(q_{\mathcal A}^\star\) is not directly available during training, QuoVLA approximates it with a quantized module, while the action-generation objective encourages the resulting representation to remain action-sufficient.

\subsection{Generalization Benefit of the Action Quotient}
\label{sec:generalization_benefit_action_quotient}

The quotient in Theorem~\ref{thm:minimal_vlm_action_quotient} removes latent
distinctions that do not change the Bayes-optimal action law. We now connect
this action-level minimality to generalization. The key point is that redundant
VLM variations can be amplified by the action head into spurious trajectory
differences, increasing the effective complexity of the learned policy. Let \(\mathcal R\) denote a representation class used by the action head, and
let \(\mathcal E_{\mathrm{act}}(\mathcal R)\) be the action-sensitive complexity
defined in Appendix~\ref{app:proof_action_sensitivity_generalization}. This
term measures how much variation in the representation can be translated into
variation of the predicted action trajectories. The appendix establishes the
following bound.

\begin{theorem}[Generalization bound of the action quotient]
\label{thm:action_quotient_generalization}
Assume the action loss is \(L_\ell\)-Lipschitz and the regularity conditions in
Appendix~\ref{app:proof_action_sensitivity_generalization} hold. For a
representation class \(\mathcal R\) used by the action head, the
action-sensitive generalization bound takes the form:
\begin{equation}
\mathrm{Gen}(\mathcal R)
\le
\mathrm{Bd}(\mathcal R)
:=
B_0(\mathcal R)
+
\frac{C L_\ell}{\sqrt T}
\mathcal E_{\mathrm{act}}(\mathcal R),
\end{equation}
where \(\mathrm{Gen}(\mathcal R)\) denotes the generalization gap,
\(B_0(\mathcal R)\) collects the remaining capacity, finite-sample, and
distribution-shift terms, and
\(\mathcal E_{\mathrm{act}}(\mathcal R)\) is the action-sensitive complexity
defined in Appendix~\ref{app:proof_action_sensitivity_generalization}. Let \(\mathcal R_H\) and \(\mathcal R_Q\) denote the representation classes
induced by \(h_\theta(x)\) and \(q_{\mathcal A}^\star(h_\theta(x))\),
respectively. Then the quotient space admits
a no-larger certified generalization bound:
\begin{equation}
\mathrm{Gen}(\mathcal R_Q)
\le
\mathrm{Bd}(\mathcal R_Q)
\le
\mathrm{Bd}(\mathcal R_H).
\end{equation}

Thus, the action quotient yields a no-worse certified bound, with strict improvement when it removes nontrivial action-irrelevant variation.
\end{theorem}

Theorem~\ref{thm:action_quotient_generalization} shows that the benefit of the
quotient is not compression alone. A quotient improves the certified bound when
it is action-sufficient and reduces the action-sensitive complexity seen by the
action head. Thus, the ideal quotient
\(q_{\mathcal A}^\star(h_\theta(x))\) helps because it preserves the
Bayes-optimal action law while removing latent variations irrelevant to action
generation. QuoVLA approximates this effect with a quantized module,
and the dual-branch action objective encourages the bottleneck representation
to remain action-sufficient.

\section{Methodology}


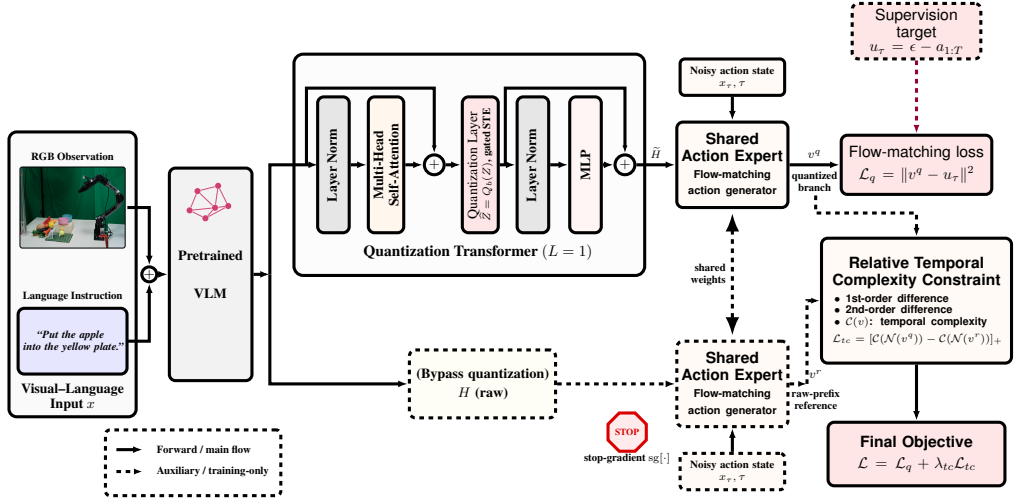
\begin{figure}
    \centering
    \resizebox{\linewidth}{!}{
    \begin{tikzpicture}[
    x=1cm,y=1cm,
    font=\sffamily,
    >=Latex,
    arrow/.style={-{Latex[length=2.2mm,width=1.6mm]},line width=2pt,draw=black},
    auxarrow/.style={-{Latex[length=2.0mm,width=1.4mm]},line width=2pt,draw=black,dashed},
    sharearrow/.style={<->,>=Latex,line width=2pt,draw=black,dashed},
    module/.style={draw=#1, fill=#1!8, rounded corners=3pt, line width=2pt,
                   align=center, inner sep=3pt, minimum height=0.72cm},
    plus/.style={circle, draw=black, fill=white, line width=2pt,
                 minimum size=0.34cm, inner sep=0pt, font=\small},
    stop/.style={regular polygon, regular polygon sides=8, draw=red!85!black,
                 fill=red!8, line width=2pt, inner sep=1pt,
                 font=\tiny\bfseries, text=red!85!black},
    note/.style={font=\scriptsize\itshape, text=black!65, align=left}
]

\path[fill=white] (0,0) rectangle (24,13.5);


\node[draw=black, fill=black!2, rounded corners=5pt, line width=2pt,minimum width=2.9cm, minimum height=6.5cm] (inputbox) at (1.75,7.60) {};

\node[font=\scriptsize\bfseries] at (1.75,10.3) {RGB Observation};

\begin{scope}
  \path[clip, rounded corners=1.5pt] (0.55,8.20) rectangle (2.95,10.05);
  \node[inner sep=0pt, anchor=center] at (1.75,9.125) {\includegraphics[width=2.40cm,height=1.85cm]{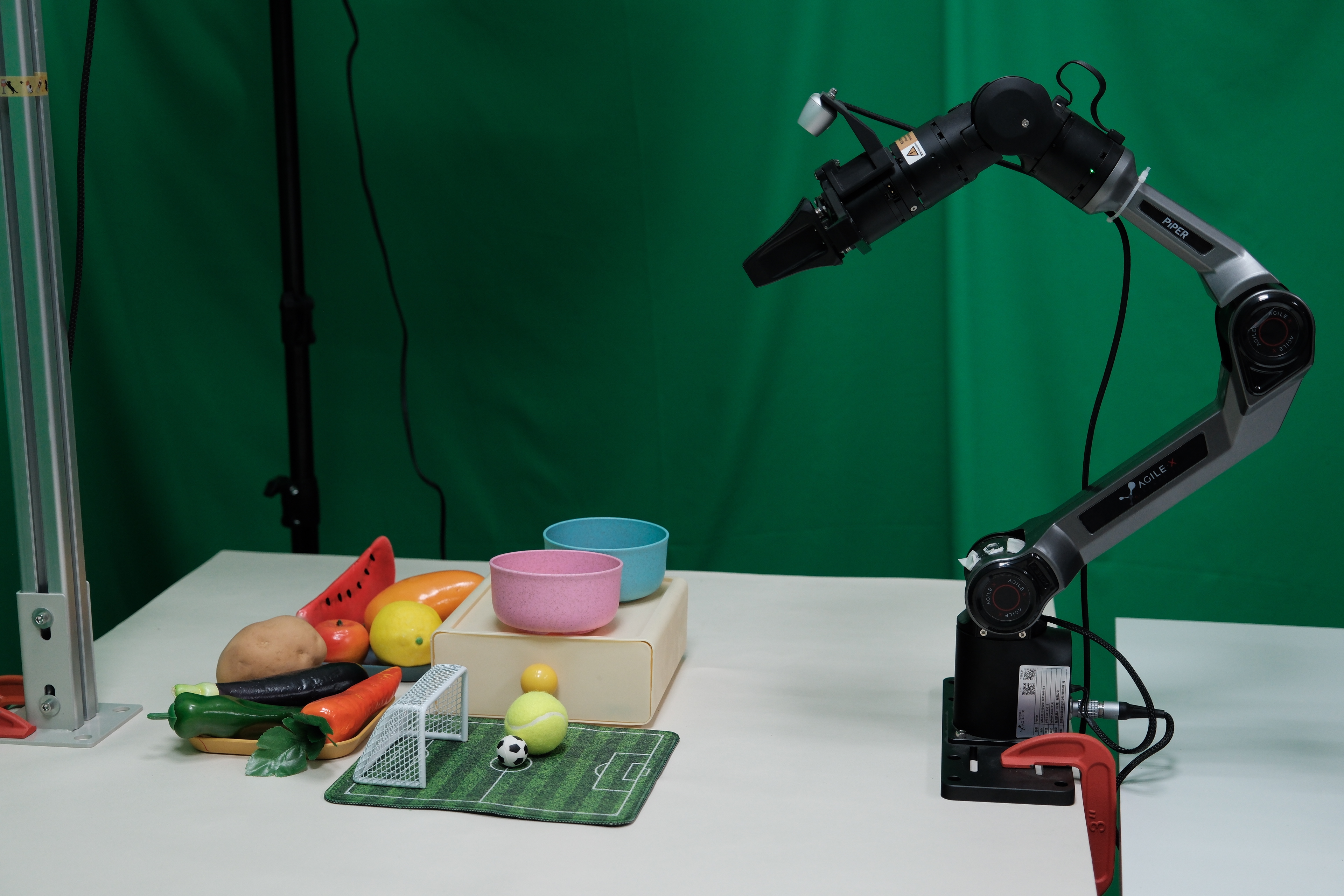}};
\end{scope}

\draw[draw=black, rounded corners=1.5pt]
  (0.55,8.20) rectangle (2.95,10.05);

\node[font=\scriptsize\bfseries] at (1.75,7.1) {Language Instruction};

\node[draw=black,fill=blue!10, rounded corners=4pt, line width=2pt,text width=2.25cm, minimum height=1.50cm, align=center,
font=\scriptsize\bfseries\itshape] (speech) at (1.75,6.075)
{\textbf{``Put the apple\\into the yellow plate.''}};



\node[font=\bfseries\small, align=center] at (1.75,4.80)
{Visual--Language\\Input $x$};

\node[plus] (merge) at (3.5,7.60) {$+$};
\draw[arrow] (3,9.125) -| (merge.north);
\draw[arrow] (3,6.075) -| (merge.south);

\node[draw=black, fill=black!5, rounded corners=4pt,
      line width=2pt, text width=1.65cm, minimum height=4.85cm,
      align=center, font=\bfseries\small] (vlm) at (4.9,7.60)
{\vspace{0.5cm}Pretrained\\VLM};

\draw[arrow] (merge.east) -- (vlm.west);

\begin{scope}[shift={(-0.25,-0.15)}]
\foreach \p/\x/\y in {
a/4.85/9.65,
b/5.32/9.77,
c/5.55/9.31,
d/5.03/9.03,
e/4.67/9.21,
f/5.25/9.27
}
\node[circle, fill=purple!75, inner sep=1.4pt] (\p) at (\x,\y) {};

\foreach \u/\v in {a/b,b/c,c/d,d/e,e/a,a/f,b/f,c/f,d/f,e/f}
\draw[purple!70,line width=0.55pt] (\u)--(\v);
\end{scope}



\coordinate (split) at (6.25,7.60);
\draw[arrow] (vlm.east) -- (split);


\begin{scope}[
shift={(1.5,0)},
rotate around={90:(9.40,10.1)},
transform shape,
qtframe/.style={
draw=black,
fill=black!1,
rounded corners=6pt,
line width=2pt
},
qttitle/.style={
font=\bfseries\color{black}
},
qnode/.style={
draw=black,
rounded corners=4pt,
line width=2pt,
minimum width=3.10cm,
minimum height=0.72cm,
align=center,
font=\small\bfseries,
text=black
},
qplus/.style={
circle,
draw=black,
fill=black!2,
line width=2pt,
inner sep=1.2pt,
font=\small\bfseries,
text=black
},
qquant/.style={
draw=black,
fill=red!10,
rounded corners=4pt,
line width=2pt,
minimum width=3.10cm,
minimum height=0.82cm,
align=center,
font=\small
}
]

\node[qtframe, minimum width=5.05cm, minimum height=8cm]
(qtbg) at (9.40,10.2) {};

\node[qttitle, rotate=-90] at (7.40,10.)
{Quantization Transformer $(L=1)$};

\node[qnode, fill=black!10] (ln1) at (9.40,13.3)
{Layer Norm};

\node[qnode, fill=orange!7] (attn) at (9.40,12.1)
{Multi-Head\\Self-Attention};

\node[qplus] (add1) at (9.40,11) {$+$};

\node[qquant] (qlayer) at (9.40,9.95)
{Quantization Layer\\[-0.5mm]{\scriptsize\bfseries $\widehat Z=Q_b(Z)$, gated STE}};

\node[qnode, fill=black!10] (ln2) at (9.40,8.75)
{Layer Norm};

\node[qnode, fill=red!4] (mlp) at (9.40,7.55)
{MLP};

\node[qplus] (add2) at (9.40,6.65) {$+$};

\coordinate (qin) at (9.40,11.08);

\coordinate (qin)  at ($(ln1.north)+(0,0.22)$);
\coordinate (qout) at ($(add2.south)+(0,-0.28)$);

\draw[arrow] (ln1.south) -- (attn.north);
\draw[arrow] (attn.south) -- (add1.north);
\draw[arrow] (add1.south) -- (qlayer.north);
\draw[arrow] (qlayer.south) -- (ln2.north);
\draw[arrow] (ln2.south) -- (mlp.north);
\draw[arrow] (mlp.south) -- (add2.north);

\draw[-{Latex[length=1.8mm,width=1.2mm]},
      line width=2pt,
      draw=black]
(qin) -- ++(1.75,0) |- (add1.east);

\coordinate (qskip) at ($(qlayer.south)!0.35!(ln2.north)$);

\draw[-{Latex[length=1.8mm,width=1.2mm]},
      line width=2pt,
      draw=black]
(qskip) -- ++(1.75,0) |- (add2.east);

\end{scope}

\draw[arrow] (split) |- (qtbg.north) |- (7.35,10.1);

\node[draw=black, fill=yellow!4, rounded corners=4pt,dashed, line width=2pt, text width=3.10cm,minimum height=1.55cm, align=center, font=\small\bfseries] (bypass) at (11.1,5.1)
{(Bypass quantization)\\[1mm]$H$ (raw)};

\draw[arrow] (split) |- (bypass.west);

\node[draw=black, fill=orange!4, rounded corners=4pt,line width=2pt, text width=2.30cm, minimum height=1.95cm,align=center] (qexpert) at (16.80,10.1)
{{\bfseries Shared\\Action Expert}\\[1mm]
{\scriptsize\bfseries Flow-matching\\action generator}};

\node[draw=black, fill=orange!4, rounded corners=3pt,
      line width=2pt, text width=2.15cm, minimum height=0.83cm,
      align=center, font=\scriptsize\bfseries] (qnoise) at (16.8,12.1)
{{Noisy action state\\$x_\tau,\tau$}};

\draw[arrow] (qnoise.south) -- (qexpert.north);

\node[draw=black, fill=orange!4, rounded corners=4pt,
      dashed, line width=2pt, text width=2.30cm, minimum height=1.95cm,
      align=center] (rexpert) at (16.80,5.1)
{{\bfseries Shared\\Action Expert}\\[1mm]
{\scriptsize\bfseries Flow-matching\\action generator}};

\node[draw=black, fill=orange!4, rounded corners=3pt,
      dashed, line width=2pt, text width=2.15cm, minimum height=0.83cm,
      align=center, font=\scriptsize\bfseries] (rnoise) at (16.80,3.1)
{{Noisy action state\\$x_\tau,\tau$}};

\draw[arrow] (rnoise.north) -- (rexpert.south);

\draw[sharearrow] (16.8,9) --
node[left,font=\scriptsize\bfseries,text=black,align=right]{shared\\weights}
(16.8,6.2);
\draw[arrow] (add2.south |- qexpert.west) --
node[above,font=\scriptsize\bfseries,align=center]
{$\widetilde H$}
(qexpert.west);

\draw[arrow, draw=black, dashed] (bypass.east) -- (rexpert.west);

\node[draw=black, fill=red!10, rounded corners=4pt,
      dashed, line width=2pt, text width=2.70cm,
      minimum height=1.05cm, align=center] (target) at (21.00,13.15)
{Supervision target\\[-0.5mm]$u_\tau=\epsilon-a_{1:T}$};

\node[draw=black, fill=red!10, rounded corners=4pt,
      line width=2pt, text width=3.20cm,
      minimum height=1.35cm, align=center] (lq) at (21.00,10.1)
{Flow-matching loss\\[1mm]$\mathcal L_q=\|v^q-u_\tau\|^2$};

\draw[auxarrow, draw=purple!80!black] (target.south) -- (lq.north);

\draw[arrow] (qexpert.east) --
node[above,font=\footnotesize\bfseries] {$v^q$}
node[below,font=\scriptsize\bfseries,align=center]{quantized\\branch}
(lq.west);

\node[draw=black, fill=orange!4, rounded corners=4pt,
      line width=2pt, text width=4.10cm,
      minimum height=2.85cm, align=center] (tc) at (21,7)
{{\bfseries Relative Temporal\\Complexity Constraint}\\[1mm]
\begin{minipage}{3.75cm}\scriptsize\bfseries
$\bullet$ 1st-order difference\\[-0.2mm]
$\bullet$ 2nd-order difference\\[-0.2mm]
$\bullet$ $\mathcal C(v)$: temporal complexity
\end{minipage}\\[1mm]
{\scriptsize\bfseries $\mathcal L_{tc}=[\mathcal C(\mathcal N(v^q))-\mathcal C(\mathcal N(v^r))]_+$}};

\draw[auxarrow] (18.7,9.4) -| (18.7,9) -| (tc.north);

\draw[auxarrow] (rexpert.east) -- ++(0.4,0) coordinate (tmp);
\draw[auxarrow] (tmp) |- (tc.west);
\node[above,font=\footnotesize\bfseries] at ($(tmp)+(0.2cm,0)$) {$v^r$};
\node[below,font=\scriptsize\bfseries,align=center] at ($(tmp)+(0.2cm,0)$) {raw-prefix\\reference};

\node[stop] (stop) at (14.4,4) {STOP};

\node[font=\scriptsize\bfseries, text=black, align=center] at (14.4,3.4)
{stop-gradient $\mathrm{sg}[\cdot]$};

\node[draw=black, fill=red!10, rounded corners=4pt,
      line width=2.0pt, text width=3.80cm,
      minimum height=1.45cm, align=center] (obj) at (21,3.5)
{{\bfseries Final Objective}\\[1mm]
$\mathcal L=\mathcal L_q+\lambda_{tc}\mathcal L_{tc}$};

\draw[arrow, draw=black, line width=2.0pt] (tc.south) -- (obj.north);






\begin{scope}[shift={(2,1.5)}]  
  \node[draw=black, rounded corners=3pt, dashed, line width=2pt,
        minimum width=4.cm, minimum height=1.5cm] (legend) at (2.5,1.80) {};

  \draw[arrow] (0.75,2.10) -- (1.35,2.10);
  \node[anchor=west,font=\scriptsize\bfseries] at (1.55,2.10) {Forward / main flow};

  \draw[auxarrow] (0.75,1.60) -- (1.35,1.60);
  \node[anchor=west,font=\scriptsize\bfseries] at (1.55,1.60) {Auxiliary / training-only};

\end{scope}
\end{tikzpicture}}
\vspace{-1.9cm}
\caption{
\textbf{Overview of the QuoVLA architecture.}
QuoVLA compresses overcomplete VLM prefix latents into an action-minimal quotient via lightweight Transformer-based prefix quantization, and uses dual-branch action generation with relative temporal-complexity regularization to preserve action-sufficient information while filtering prompt-level redundancy.
}
\label{fig:model}
\end{figure}

\subsection{Transformer-based Prefix Quantization}
\label{sec:prefix_quantization}

To approximate the action-minimal quotient, we insert a lightweight quantized Transformer block between the VLM prefix encoder and the action expert.
Given a visual-language input \(x\), the VLM produces prefix token embeddings
\begin{equation}
H
=
h_\theta(x)
=
[h_1,\ldots,h_M]
\in
\mathbb R^{M\times d},
\end{equation}
where \(M\) is the number of prefix tokens and \(d\) is the hidden dimension. Unlike a standard Transformer block, our block places the quantization layer between the attention sublayer and the MLP sublayer:
\begin{align}
H^{\mathrm{att}}
&=
\mathrm{LN}
\left(
H+\mathrm{MHA}_{\psi}(H)
\right),\\
\widehat H
&=
Q_b\!\left(H^{\mathrm{att}}\right),\\
\widetilde H
&=
\mathrm{LN}
\left(
\widehat H+\mathrm{MLP}_{\psi}(\widehat H)
\right).
\end{align}
Here \(Q_b\) is applied to the attention-updated prefix before it enters the MLP, and \(\widetilde H\) denotes the output prefix of the quantized Transformer block. Let \(Q_b(\cdot)\) be a \(b\)-bit per-token activation quantizer. For a token vector \(z\), we use symmetric uniform quantization:
\begin{equation}
Q_b(z)
=
s\cdot
\mathrm{clip}
\left(
\mathrm{round}\left(\frac{z}{s}\right),
-q_{\max},
q_{\max}
\right),
\qquad
q_{\max}=2^{b-1}-1,
\end{equation}
where $s=\frac{\|z\|_{\infty}}{q_{\max}}$. Since \(Q_b\) contains rounding, its true derivative is zero almost everywhere. We therefore use an adaptive straight-through estimator. Let \(\alpha\) be a trainable gate parameter and define:
\begin{equation}
g
=
g_{\min}
+
(1-g_{\min})\sigma(\alpha),
\qquad
g\in[g_{\min},1],
\end{equation}
where \(\sigma(\cdot)\) is the sigmoid function. We implement the quantized attention activation as
\begin{equation}
\widehat H
=
\operatorname{sg}\!\left(
Q_b(H^{\mathrm{att}})-gH^{\mathrm{att}}
\right)
+
gH^{\mathrm{att}},
\end{equation}
where \(\operatorname{sg}[\cdot]\) denotes stop-gradient. This formulation keeps the forward pass identical to quantization:
\begin{equation}
\widehat H
=
Q_b(H^{\mathrm{att}}),
\end{equation}
while the backward pass uses the gated surrogate gradient:
\begin{equation}
\frac{\partial \widehat H}{\partial H^{\mathrm{att}}}
\approx
gI.
\end{equation}
Thus, gradients from action supervision pass through the attention-to-MLP quantization bottleneck with a learned strength controlled by \(\alpha\). The action expert predicts the action trajectory conditioned on the quantized-block output:
\begin{equation}
\pi_{\phi}^{q}(a_{1:T}\mid x)
=
G_{\phi}^{\mathrm{act}}
\left(
a_{1:T}\mid \widetilde H
\right).
\end{equation}
This design discretizes the intermediate prefix activations before MLP processing, encouraging the model to suppress action-redundant visual-language variations while preserving information needed for action generation.

\subsection{Dual-Branch Action Generation}
\label{sec:dual_branch_action_generation}

Given the raw VLM prefix \(H\) and the quantized prefix \(\widetilde H\), we use two action-generation branches.
The quantized branch is the main training branch, while the non-quantized branch serves as a stop-gradient reference. Following flow matching, for a ground-truth action trajectory \(a_{1:T}\), we sample noise \(\epsilon\) and time \(\tau\), and construct:
\begin{equation}
x_\tau
=
\tau \epsilon
+
(1-\tau)a_{1:T},
\qquad
u_\tau
=
\epsilon-a_{1:T}.
\end{equation}
The quantized branch predicts the velocity field
\begin{equation}
v^q
=
G_\phi^{\mathrm{act}}
\left(
x_\tau,\tau\mid \widetilde H
\right),
\end{equation}
and is trained by the flow-matching loss
\begin{equation}
\mathcal L_q
=
\left\|
v^q-u_\tau
\right\|_2^2 .
\end{equation}

The non-quantized branch predicts a reference velocity field from the raw prefix:
\begin{equation}
v^r
=
\mathrm{sg}
\left[
G_\phi^{\mathrm{act}}
\left(
x_\tau,\tau\mid H
\right)
\right],
\end{equation}
where \(\mathrm{sg}[\cdot]\) denotes stop-gradient.
Thus, the raw branch does not receive gradients from the branch constraint and only provides a reference for the quantized branch.

To prevent quantization from producing overly high-frequency action velocities, we impose a relative temporal-complexity constraint.
For a velocity trajectory \(v\in\mathbb R^{T\times D}\), we define
\begin{equation}
\mathcal C(v)
=
\lambda_1
\frac{1}{(T-1)D}
\sum_{t=1}^{T-1}
\left\|
v_{t+1}-v_t
\right\|_2^2
+
\lambda_2
\frac{1}{(T-2)D}
\sum_{t=1}^{T-2}
\left\|
v_{t+2}-2v_{t+1}+v_t
\right\|_2^2 .
\end{equation}
The first-order term penalizes rapid changes between adjacent velocity steps, while the second-order term penalizes sharp curvature and high-frequency oscillations.
Therefore, \(\mathcal C(v)\) measures the temporal complexity of the predicted velocity field rather than its absolute magnitude. The relative constraint is:
\begin{equation}
\mathcal L_{\mathrm{tc}}
=
\left[
\mathcal C\bigl(\mathcal N(v^q)\bigr)
-
\mathcal C\bigl(\mathcal N(v^r)\bigr)
\right]_+,
\end{equation}
where \([\cdot]_+=\max(\cdot,0)\), and \(\mathcal N(\cdot)\) normalizes action dimensions before computing temporal differences.
This loss penalizes the quantized branch only when its velocity field is more temporally complex than the stop-gradient raw-prefix reference. The final training objective is:
\begin{equation}
\mathcal L
=
\mathcal L_q
+
\lambda_{\mathrm{tc}}\mathcal L_{\mathrm{tc}} .
\end{equation}
This dual-branch design preserves action generation from the quantized prefix while constraining it not to introduce unnecessary temporal complexity beyond the raw VLM reference.

\section{Experiments}

\subsection{Experimental Setting}

\paragraph{Simulation Benchmarks.}
We evaluate QuoVLA on four simulation benchmarks: \textit{LIBERO}~\cite{liu2023libero}, \textit{LIBERO-PRO}~\cite{zhou2025libero}, \textit{LIBERO-Plus}~\cite{fei2025libero}, and \textit{RoboTwin~2.0}~\cite{chen2025robotwin}. We use the official success predicate of each environment and report the task success rate averaged over the evaluated tasks. These benchmarks provide complementary coverage of standard language-conditioned manipulation, generalization under distribution shifts, fine-grained robustness to nuisance perturbations, and diverse bimanual manipulation. Detailed benchmark suites and perturbations are provided in Appendix~\ref{app:sim_benchmark_details}.

\paragraph{Real-robot Setup.}
We further evaluate QuoVLA on four real-robot manipulation tasks: moving a tennis ball from a yellow plate to a blue plate, placing a red cube into a yellow plate, putting an apple on a yellow plate, and removing a cuboid from a blue plate. For each task, we use 100 training samples and train the policy on a single NVIDIA A800 GPU for 5{,}000 steps.

\paragraph{Baselines.}
We take $\pi_{0.5}$~\cite{intelligence2025pi_05} as the primary baseline, since QuoVLA is built on top of it. We compare with Diffusion Policy~\cite{chi2025diffusion}, OpenVLA~\cite{kim2024openvla}, SpatialVLA~\cite{qu2025spatialvla}, CoT-VLA~\cite{zhao2025cot}, UniVLA~\cite{bu2025univla}, OpenVLA-OFT~\cite{kim2025fine}, MemoryVLA~\cite{shi2025memoryvla}, $\pi_{0}$~\cite{black2024pi_0}, VLA-Adapter~\cite{wang2026vla}, ABot-M0~\cite{yang2026abot}, Being-H0.5~\cite{luo2026being}, FLOWER~\cite{reuss2025flower}, X-VLA~\cite{zheng2025x}, VOTE~\cite{lin2025vote}, and SimVLA~\cite{luo2026simvla}.

\paragraph{Evaluation Protocol.}
We report average task success rates using each benchmark's official success predicate. Following recent VLA evaluations \cite{zheng2025x,liang2025discrete}, we train one generalist policy on the union of LIBERO Spatial, Object, Goal, and Long suites and evaluate it on the original tasks. We further test the same LIBERO-trained policy on LIBERO-PRO for zero-shot robustness, and report both zero-shot transfer and in-domain fine-tuning on LIBERO-Plus. For RoboTwin~2.0, we fine-tune on its training tasks and follow the official protocol.

\paragraph{Implementation Details.}
QuoVLA follows the $\pi_{0.5}$ setup unless otherwise specified, using the PaliGemma-2.6B VLM and a 300M action expert. Images are resized to $224 \times 224$; the policy takes one visual-state-language step and predicts a 50-step action chunk, with states and actions padded to 32 dimensions and quantile-normalized. For quotient learning, we use 8-bit prefix quantization with a one-layer, eight-head Transformer and train both quantized and non-quantized action routes. Models are optimized with AdamW using a $2.5\times10^{-5}$ learning rate, $0.01$ weight decay, $1.0$ gradient clipping, and cosine decay with warmup. All simulation runs use 8 NVIDIA A800-80GB GPUs. We train for approximately 10 epochs on the LIBERO-family datasets and for only 5 epochs on RoboTwin~2.0.

\subsection{Main Results}

\begin{table}[t]
\centering
\caption{
Results on LIBERO and LIBERO-PRO. Left: LIBERO success rates (\%), with best/second-best marked by \textbf{bold}/\underline{underline}; right: LIBERO-PRO robustness success rates (\%), with best in \textbf{bold}.
}
\label{tab:libero_and_pro}

\begin{minipage}[t]{0.47\textwidth}
\centering
\textbf{(a) LIBERO benchmark}
\scriptsize
\setlength{\tabcolsep}{2.6pt}
\renewcommand{\arraystretch}{1.05}
\resizebox{\ifdim\width>\linewidth\linewidth\else\width\fi}{!}{%
\begin{tabular}{@{}lccccc@{}}
\toprule
\textbf{Model} & \textbf{Spatial} & \textbf{Object} & \textbf{Goal} & \textbf{Long} & \textbf{Avg} \\
\midrule

Diffusion Policy \cite{chi2025diffusion} & 78.5& 87.5 &73.5& 64.8 &76.1\\
OpenVLA \cite{kim2024openvla}  &84.7 &88.4& 79.2& 53.7& 76.5\\
SpatialVLA \cite{qu2025spatialvla}  &88.2 &89.9& 78.6& 55.5& 78.1\\
CoT-VLA \cite{zhao2025cot}  &87.5& 91.6& 87.6& 69.0& 83.9\\

\hline
UniVLA~\cite{bu2025univla}        & 96.5 & 96.8 & 95.6 & 92.0 & 95.2 \\
OpenVLA-OFT~\cite{kim2025fine} & 97.6 & 98.4 & 97.9 & 94.5 & 97.1 \\
MemoryVLA~\cite{shi2025memoryvla} & 98.4 & 98.4 & 96.4 & 93.4 & 96.7 \\
$\pi_{0}$~\cite{black2024pi_0}     & 96.8 & 98.8 & 95.8 & 85.2 & 94.2 \\
$\pi_{0.5}$~\cite{intelligence2025pi_05}  & {98.8} & 98.2 & {98.0} & 92.4 & 96.9 \\
VLA-Adapter~\cite{wang2026vla}
                                  & 97.8 & {99.2} & 97.2 & 95.0 & 97.3 \\
ABot-M0 \cite{yang2026abot} & 98.8 &\underline{99.8} &\underline{99.0}& {96.6}& \underline{98.6}\\
Being-H0.5 \cite{luo2026being} & 97.0& 98.2 &\underline{99.0}& 96.2 &97.6\\

FLOWER \cite{reuss2025flower} & 97.1 &96.7 &95.6 &93.5 &95.7\\
X-VLA \cite{zheng2025x}& 98.2& 98.6& 97.8& \underline{97.6}& 98.1\\

VOTE \cite{lin2025vote}& 98.8& \underline{99.8}& 97.6& 95.6& 98.0\\

SimVLA  \cite{luo2026simvla}  & \underline{99.6}& \underline{99.8} &98.6 & 96.4 & \underline{98.6} \\
\midrule

\rowcolor{pink!30}
QuoVLA& \textbf{99.8} &\textbf{99.9}&\textbf{100.0}&\textbf{98.7}&\textbf{99.6}\\
\bottomrule
\end{tabular}%
}
\end{minipage}
\hfill
\begin{minipage}[t]{0.51\textwidth}
\centering
\textbf{(b) LIBERO-PRO benchmark}

\scriptsize
\setlength{\tabcolsep}{2.4pt}
\renewcommand{\arraystretch}{1.05}
\resizebox{\ifdim\width>\linewidth\linewidth\else\width\fi}{!}{%
\begin{tabular}{@{}l|l|ccccc@{}}
\toprule
\textbf{Suite} & \textbf{Model} & \textbf{Ori} & \textbf{Obj} & \textbf{Pos} & \textbf{Sem} & \textbf{Task} \\
\midrule

\multirow{3}{*}{Spatial}
& OpenVLA~\cite{kim2024openvla}      & 98.0 & 97.0 & 0.0  & 97.0 & 0.0 \\
& $\pi_{0.5}$~\cite{intelligence2025pi_05}   & 98.0 & 97.0 & 20.0 & 97.0 & {1.0} \\
&SimVLA \cite{luo2026simvla}                  & {99.0} & {98.0} & {29.0} & {98.0} & 0.0 \\
\rowcolor{pink!30}
&QuoVLA   & \textbf{100.0} & \textbf{99.0} & \textbf{40.0} & \textbf{99.0} & \textbf{10.0} \\
\midrule

\multirow{3}{*}{Object}
& OpenVLA~\cite{kim2024openvla}      & 99.0 & {98.0} & 0.0  & 98.0 & 0.0 \\
& $\pi_{0.5}$~\cite{intelligence2025pi_05}   & 98.0 & \textbf{98.0} & {17.0} & 96.0 & 1.0 \\
& SimVLA \cite{luo2026simvla} & {100.0} & 85.0 & 1.0 & {100.0} & {4.0} \\
\rowcolor{pink!30}
& QuoVLA    & \textbf{100.0} & 87.0 & \textbf{26.0} & \textbf{100.0} & \textbf{26.0} \\
\midrule

\multirow{3}{*}{Goal}
& OpenVLA~\cite{kim2024openvla}      & 98.0 & 96.0 & 0.0  & 98.0 & 0.0 \\
& $\pi_{0.5}$~\cite{intelligence2025pi_05}   & 97.0 & \textbf{97.0} & {38.0} & 97.0 & 0.0 \\
& {SimVLA}   & {99.0} & 82.0 & 0.0 & {99.0} & {10.0} \\

& QuoVLA   & \textbf{100.0} & 85.0 & 57.0 & \textbf{100.0} & \textbf{42.0} \\

\midrule

\multirow{3}{*}{Long}
& OpenVLA~\cite{kim2024openvla}      & 93.0 & 81.0 & 0.0 & 96.0 & 0.0 \\
& $\pi_{0.5}$~\cite{intelligence2025pi_05}   & 93.0 & \textbf{92.0} & {8.0} & 93.0 & 1.0 \\
& SimVLA \cite{luo2026simvla}                  & {96.0} & 61.0 & 3.0 & {98.0} & {10.0} \\

\rowcolor{pink!30}
& QuoVLA     & \textbf{99.0} & 82.0 & \textbf{21.0} & \textbf{99.0} & \textbf{24.0} \\

\bottomrule
\end{tabular}%
}
\end{minipage}

\end{table}

\paragraph{Overall performance.}
As shown in Tables~\ref{tab:libero_and_pro},\ref{tab:libero_plus} and \ref{tab:robotwin_subset_quovla}, QuoVLA achieves consistently strong results across all simulation benchmarks. It reaches the best average success rate on LIBERO, obtains the best or tied-best performance in 17 out of 20 LIBERO-PRO robustness settings, and achieves strong results on LIBERO-Plus under both zero-shot transfer and supervised fine-tuning. On RoboTwin~2.0, QuoVLA improves Hard-setting robustness over $\pi_{0.5}$ from 43.84\% to 58.6\% while maintaining competitive Easy-setting performance. These results indicate that the quotient-based action representation improves robustness to distribution shifts by filtering task-irrelevant visual-language variations while preserving action-discriminative information for control.

\begin{table}[t]
\caption{
Comparison on LIBERO-Plus. \textit{Zero-Shot Transfer} directly evaluates LIBERO-trained models on LIBERO-Plus, while \textit{Supervised Fine-Tuning} fine-tunes them on the LIBERO-Plus training split. Best results are in \textbf{bold}.
}
\centering
\scriptsize
\setlength{\tabcolsep}{4.0pt}
\renewcommand{\arraystretch}{1.05}
\resizebox{0.75\linewidth}{!}{%
\begin{tabular}{@{}l|ccccccc|c@{}}
\toprule
\textbf{Models} 
& \textbf{Camera} 
& \textbf{Robot} 
& \textbf{Language} 
& \textbf{Light} 
& \textbf{Background} 
& \textbf{Noise} 
& \textbf{Layout} 
& \textbf{Avg.} \\
\midrule
\multicolumn{9}{c}{\textit{\textbf{Zero-Shot Transfer}}} \\
\midrule

OpenVLA~\cite{kim2024openvla} 
& 0.8 & 3.5 & 23.0 & 8.1 & 34.8 & 15.2 & 28.5 & 15.6 \\

UniVLA~\cite{bu2025univla} 
& 1.8 & 46.2 & 69.6 & 69.0 & 81.0 & 21.2 & 31.9 & 42.9 \\

OpenVLA-OFT~\cite{kim2025fine} 
& 56.4 & 31.9 & 79.5 & 88.7 & 93.3 & 75.8 & 74.2 & 69.6 \\

$\pi_{0}$~\cite{black2024pi_0} 
& 61.0 & 40.8 & 63.5 & 89.3 & 84.1 & 80.1 & 76.4 & 69.4 \\

$\pi_{0.5}$~\cite{intelligence2025pi_05} 
& {75.8} & 79.4 & 83.3 & 95.5 & 95.0 & {89.6} & 87.0 & 85.7 \\

ABot-M0 \cite{yang2026abot} & 60.4 &67.9& 86.4& 96.2 &91.6 &86.4 &82.6& 80.5\\

VLA-JEPA \cite{sun2026vla}&63.3& 67.1 &85.4& 95.6& 93.6& 66.3& 85.1& 79.5\\

ACoT-VLA \cite{zhong2026acot}
& 72.6 & {82.6} & {87.5} & {97.7} & \textbf{96.5} & 87.8 & 88.1 & {86.6} \\

\rowcolor{pink!30}
\textbf{Ours} 
& \textbf{82.3} & \textbf{87.6} & \textbf{88.2} & \textbf{98.2} & {95.9} & \textbf{90.8} & \textbf{89.3} & \textbf{90.3} \\

\midrule
\multicolumn{9}{c}{\textit{\textbf{Supervised Fine-Tuning}}} \\
\midrule

$\pi_{0}$~\cite{black2024pi_0} 
& 79.6 & 21.1 & 72.5 & 84.7 & 86.2 & 68.3 & 69.4 & 67.4 \\

$\pi_{0.5}$~\cite{intelligence2025pi_05} 
& 70.3 & 41.7 & {81.1} & \textbf{97.3} & 94.6 & 71.8 & 84.9 & 75.7 \\

ACoT-VLA \cite{zhong2026acot}
& {96.6} & {70.4} & 79.7 & 95.1 & {97.1} & \textbf{95.9} & {85.0} & {88.0} \\

\rowcolor{pink!30}
\textbf{Ours} 
& \textbf{98.4} & \textbf{73.5} & \textbf{83.7} & 96.3 & \textbf{99.0} & 94.8 & \textbf{90.4} & \textbf{90.9} \\

\bottomrule
\end{tabular}%
}

\label{tab:libero_plus}
\end{table}

\colorlet{bestpurple}{pink!30}
\newcommand{\hl}[1]{\cellcolor{bestpurple}#1}

\begin{table*}[t]
\centering
\caption{Subset of RoboTwin 2.0 benchmark.}
\label{tab:robotwin_subset_quovla}
\setlength{\tabcolsep}{4pt}
\renewcommand{\arraystretch}{1.08}
\resizebox{0.75\textwidth}{!}{
\begin{tabular}{lcccccccccc}
\toprule
Simulation Task
& \multicolumn{2}{c}{ACT}
& \multicolumn{2}{c}{$\pi_{0}$}
& \multicolumn{2}{c}{$\pi_{0.5}$}
& \multicolumn{2}{c}{DP3}
& \multicolumn{2}{c}{QuoVLA} \\
\cmidrule(lr){2-3}
\cmidrule(lr){4-5}
\cmidrule(lr){6-7}
\cmidrule(lr){8-9}
\cmidrule(lr){10-11}
& Easy & Hard
& Easy & Hard
& Easy & Hard
& Easy & Hard
& Easy & Hard \\
\midrule

\textit{Adjust Bottle}
& 97\% & 23\%
& 90\% & 56\%
& 12\% & 7\%
& \hl{99\%} & 3\%
& 18\% & \hl{68\%} \\

\textit{Beat Block Hammer}
& 56\% & 3\%
& 43\% & 21\%
& 16\% & 18\%
& \hl{72\%} & 8\%
& 20\% & \hl{42\%} \\

\textit{Blocks Ranking RGB}
& 1\% & 0\%
& 19\% & 5\%
& 48\% & 56\%
& 3\% & 0\%
& \hl{50\%} & \hl{62\%} \\

\textit{Blocks Ranking Size}
& 0\% & 0\%
& 7\% & 1\%
& 42\% & 38\%
& 2\% & 0\%
& \hl{43\%} & \hl{45\%} \\

\textit{Click Alarmclock}
& 32\% & 4\%
& 63\% & 11\%
& 74\% & 65\%
& \hl{77\%} & 14\%
& 75\% & \hl{72\%} \\

\textit{Click Bell}
& 58\% & 3\%
& 44\% & 3\%
& 25\% & 36\%
& \hl{90\%} & 0\%
& 27\% & \hl{48\%} \\

\textit{Dump Bin Bigbin}
& 68\% & 1\%
& 83\% & 24\%
& 33\% & 29\%
& \hl{85\%} & 53\%
& 35\% & \hl{60\%} \\

\textit{Grab Roller}
& 94\% & 25\%
& 96\% & 80\%
& 10\% & 6\%
& \hl{98\%} & 2\%
& 14\% & \hl{86\%} \\

\textit{Handover Block}
& 42\% & 0\%
& 45\% & 8\%
& 43\% & 35\%
& \hl{70\%} & 0\%
& 45\% & \hl{50\%} \\

\addlinespace[0.6em]
\multicolumn{11}{c}{$\cdots$} \\
\addlinespace[0.6em]

\textit{Move Pillbottle Pad}
& 0\% & 0\%
& 21\% & 1\%
& 35\% & 37\%
& \hl{41\%} & 0\%
& 36\% & \hl{45\%} \\

\midrule
\textit{Average (in \%)}
& 29.7 & 1.7
& 46.4 & 16.3
& 42.98 & 43.84
& \hl{55.2} & 5.0
& 45.1 & \hl{58.6} \\

\bottomrule
\end{tabular}
}
\end{table*}

\begin{table}[t]
\centering
\caption{
Real-robot evaluation results. We report the task success rate (\%), and $\Delta$ denotes the absolute improvement of QuoVLA over $\pi_{0.5}$.
}
\label{tab:real_robot_results}
\small
\setlength{\tabcolsep}{5pt}
\begin{tabularx}{0.7\linewidth}{Xccc}
\toprule
Real-robot Task & $\pi_{0.5}$ & QuoVLA & $\Delta$ \\
\midrule
Move tennis ball from yellow plate to blue plate & 93/100 & 97/100 & +4\% \\
Pick red cube into yellow plate & 48/100 & 74/100 & +26\% \\
Put apple on yellow plate & 31/100 & 83/100 & +52\% \\
Remove cuboid from blue plate & 86/100 & 98/100 & +12\% \\
\midrule
Average & 64.5\% & 88\% & +23.5\% \\
\bottomrule
\end{tabularx}
\end{table}

\paragraph{Real-robot performance.}
Table~\ref{tab:real_robot_results} summarizes real-robot success rates on four manipulation tasks. Each task uses 100 training samples and is trained for 5{,}000 steps on a single NVIDIA A800 GPU. QuoVLA improves the average success rate from 64.5\% to 88.0\%, yielding a +23.5\% gain over $\pi_{0.5}$. The improvements are especially large on precise placement tasks, such as picking the red cube into the yellow plate (+26\%) and putting the apple on the yellow plate (+52\%), while QuoVLA also improves the already strong baseline on tennis-ball transfer (+4\%) and cuboid removal (+12\%). These results show that the quotient-based action representation improves data-efficient real-robot adaptation, particularly under contact and placement uncertainty.

\subsection{Ablation Study}

\paragraph{Ablation analysis.}
Table~\ref{tab:libero_ablation} ablates the main components of QuoVLA on LIBERO. The default setting achieves the best average success rate of 99.6\%. Increasing the quantization depth from $L_q=1$ to $L_q=2$ and $L_q=6$ lowers the average performance to 97.7\% and 95.2\%, suggesting that excessive prefix processing may distort action-relevant structure. Increasing the bit-width from $b_q=8$ to $b_q=16$ also reduces the average success rate to 97.1\%, indicating that a proper quantization bottleneck is beneficial for quotienting out action-redundant prompt variations. Without Adaptive STE, performance drops slightly to 98.45\%, confirming the importance of stable gradient estimation for the discrete bottleneck. Removing the dual-branch design or its associated constraints yields 98.65\% and 98.78\%, respectively, showing that these two components complement the quantization module by preserving action-sufficient structure during action generation.

\begin{figure}[t]
\centering

\begin{minipage}[t]{0.58\textwidth}
\centering
\vspace{0pt}
\scriptsize
\captionof{table}{
Ablations on LIBERO.
Each row corresponds to one ablation setting, with all other hyperparameters fixed to their default values.
}
\label{tab:libero_ablation}
\footnotesize
\setlength{\tabcolsep}{5.0pt}
\renewcommand{\arraystretch}{1.08}

\resizebox{\linewidth}{!}{
\begin{tabular}{@{}l|l|ccccc@{}}
\toprule
\textbf{Knob} 
& \textbf{Value} 
& \textbf{Spatial} 
& \textbf{Object} 
& \textbf{Goal} 
& \textbf{Long} 
& \textbf{Avg} \\
\midrule

\rowcolor{pink!30}
\textbf{QuoVLA}
& Default settings
& \textbf{99.8}
& \textbf{99.9}
& \textbf{100.0}
& \textbf{98.7}
& \textbf{99.6} \\

\midrule
\multicolumn{7}{@{}l@{}}{
\textit{Quantization parameters}
\quad
(default: $L_q=1$, $b_q=8$, Standard STE)
} \\
\midrule

w/o Quantization & --
& 98.8 & 98.2 & 98.0 & 92.4 & 96.85\\
\midrule

\multirow{3}{*}{Quantization depth $L_q$}
& $L_q=1$ 
& 99.8 & 99.9 & 100.0& 98.7 & 99.6 \\
& $L_q=2$
& 96.4 & 96.1 & 95.6 & 94.8 & 95.7 \\
& $L_q=6$
& 93.7 & 92.9 & 91.8 & 90.5 & 92.2 \\

\midrule
\multirow{3}{*}{Quantization bit-width $b_q$}
& $b_q=4$ 
& 84.4 & 87.3 & 85.0 & 73.7 & 82.6 \\
& $b_q=8$ 
& 99.8 & 99.9 & 100.0& 98.7 & 99.6 \\
& $b_q=16$ 
& 98.3 & 98.8 & 97.2 & 94.1 & 97.1 \\

\midrule

w/o Adaptive STE & --
& 98.4 & 98.5 & 99.1 & 97.8 & 98.45 \\

\midrule
\multicolumn{7}{@{}l@{}}{
\textit{Action-generation design}
\quad
(default: dual-branch action generation with constraints)
} \\
\midrule

w/o dual-branch & --
& 98.4 & 98.6 & 99.2 & 98.4 & 98.65 \\

w/o Constraints & --
& 98.5 & 98.6 & 99.7 & 98.3 & 98.78 \\

\bottomrule
\end{tabular}
}
\end{minipage}
\hfill
\begin{minipage}[t]{0.38\textwidth}
\centering
\vspace{0pt}

\resizebox{\linewidth}{!}{
\begin{tikzpicture}
\begin{axis}[
    width=1\linewidth,
    height=0.75\linewidth,
    xlabel={Gaussian Noise Std. $\sigma$},
    ylabel={Average Success Rate (\%)},
    xmin=0,
    xmax=0.10,
    ymin=10,
    ymax=102,
    xtick={0,0.02,0.04,0.06,0.08,0.10},
    xticklabels={0,0.2,0.4,0.6,0.8,1.0},
    ytick={20,40,60,80,100},
    grid=major,
    major grid style={draw=gray!25},
    tick label style={font=\scriptsize},
    label style={font=\scriptsize},
    legend style={
        at={(0.5,-0.3)},
        anchor=north,
        legend columns=4,
        draw=none,
        font=\scriptsize
    },
    every axis plot/.append style={
        line width=1.2pt,
        mark size=2.2pt
    },
]

\addplot+[mark=o] coordinates {
    (0.00,76.5)
    (0.01,75.8)
    (0.02,73.2)
    (0.03,70.4)
    (0.04,65.7)
    (0.05,61.8)
    (0.06,52.4)
    (0.07,45.6)
    (0.08,36.1)
    (0.09,29.4)
    (0.10,19.2)
};
\addlegendentry{OpenVLA}

\addplot+[mark=square*] coordinates {
    (0.00,94.2)
    (0.01,93.6)
    (0.02,91.0)
    (0.03,88.4)
    (0.04,84.7)
    (0.05,77.9)
    (0.06,71.5)
    (0.07,63.1)
    (0.08,55.8)
    (0.09,44.6)
    (0.10,35.2)
};
\addlegendentry{$\pi_{0}$}

\addplot+[mark=triangle*] coordinates {
    (0.00,96.9)
    (0.01,96.1)
    (0.02,95.2)
    (0.03,93.0)
    (0.04,91.3)
    (0.05,87.5)
    (0.06,83.8)
    (0.07,78.6)
    (0.08,72.4)
    (0.09,66.8)
    (0.10,59.7)
};
\addlegendentry{$\pi_{0.5}$}

\addplot+[very thick, mark=star] coordinates {
    (0.00,99.6)
    (0.01,99.0)
    (0.02,98.4)
    (0.03,97.6)
    (0.04,96.2)
    (0.05,94.8)
    (0.06,92.1)
    (0.07,90.5)
    (0.08,86.7)
    (0.09,84.1)
    (0.10,80.4)
};
\addlegendentry{QuoVLA}

\end{axis}
\end{tikzpicture}
}

\captionof{figure}{
Noise robustness under increasing Gaussian corruption. We evaluate all methods on the four LIBERO suites and report the average success rate. Gaussian noise with standard deviation $\sigma$ is added to the RGB observations, where larger $\sigma$ indicates stronger visual corruption.
}
\label{fig:noise_robustness_curve}
\end{minipage}

\end{figure}

\paragraph{Perturbation-wise robustness curve.} 
As shown in Figure~\ref{fig:noise_robustness_curve}, we evaluate three baselines and our model on the four LIBERO suites and report the average success rate under increasing Gaussian noise. As the noise level increases, all methods exhibit progressively larger degradation, indicating that stronger visual corruption leads to a more severe distribution shift. Among the baselines, OpenVLA and $\pi_{0}$ deteriorate much faster, while $\pi_{0.5}$ shows relatively better robustness. In contrast, QuoVLA consistently maintains the highest curve and degrades the slowest over the entire noise range. This suggests that QuoVLA is substantially less sensitive to observation noise. The advantage mainly comes from our quotient-based action representation, which suppresses task-irrelevant visual perturbations while preserving action-discriminative information for stable action prediction.

\section{Limitations}

QuoVLA relies on a sufficiently strong pretrained VLM whose prefix representations already encode action-relevant information. When the base VLM is weak or poorly aligned with the target robot domain, quotienting can suppress redundant variations but cannot fully compensate for missing action cues. In addition, the quantization bottleneck requires proper capacity to avoid distorting useful action structure. Our real-robot evaluation is also limited to tabletop manipulation with a single hardware setup, leaving broader embodiments and longer-horizon open-world tasks for future work.

\section{Conclusion}

We presented QuoVLA, a quotient-space framework for adapting pretrained VLM representations to robot control. Our theory reframes VLA adaptation by showing that pretrained VLM latents can be action-sufficient yet overcomplete, motivating the removal of action-redundant visual-language variations. QuoVLA operationalizes this idea with transformer-based prefix quantization and dual-branch action generation with temporal-complexity regularization.

\bibliographystyle{plainnat}  
\bibliography{references}


\clearpage
\appendix

\section{Appendix}

\subsection{Simulation Benchmark Details}
\label{app:sim_benchmark_details}

We evaluate QuoVLA on four simulation benchmarks: \textit{LIBERO}~\cite{liu2023libero}, \textit{LIBERO-PRO}~\cite{zhou2025libero}, \textit{LIBERO-Plus}~\cite{fei2025libero}, and \textit{RoboTwin~2.0}~\cite{chen2025robotwin}. We use the official success predicate of each environment and report the task success rate averaged over the evaluated tasks.

\paragraph{LIBERO.}
The original \textit{LIBERO} benchmark focuses on knowledge transfer in language-conditioned manipulation. We evaluate on its standard suites, including LIBERO-Spatial, LIBERO-Object, LIBERO-Goal, and LIBERO-Long.

\paragraph{LIBERO-PRO.}
\textit{LIBERO-PRO} tests generalization beyond memorization in LIBERO-style tasks. We evaluate four splits: Object, Position, Semantic, and Task, covering visual object shifts, spatial relocation, language paraphrasing, and task-logic changes.

\paragraph{LIBERO-Plus.}
\textit{LIBERO-Plus} is used for fine-grained robustness analysis. Unlike LIBERO-PRO, which evaluates broader generalization splits, LIBERO-Plus perturbs individual nuisance factors in LIBERO tasks. We evaluate seven perturbations: object layout, camera viewpoint, robot initial state, language instruction, lighting, background texture, and sensor noise.

\paragraph{RoboTwin~2.0.}
\textit{RoboTwin~2.0} evaluates robustness beyond single-arm LIBERO tasks, covering diverse bimanual manipulation scenarios with domain randomization over clutter, lighting, background, tabletop height, and language instructions.

\subsection{Injectivity for VLA Representations}
\label{app:injectivity_vla}

\begin{knownresult}[Restated from Nikolaou et al.~\cite{nikolaou2026language}, Theorem C.1]
Fix a finite vocabulary $\mathcal V$, a context bound $K \in \mathbb N$, and a time horizon $T \in \mathbb N$. Consider a causal Transformer Language Model (TLM) with parameter space $\mathbb R^p$, output map $f(\cdot;\theta)$, and last-token, last-layer representation map $\mathbf r(\cdot;\theta)$. Let$\left\{\left(s_t \in \mathcal V^{\le K},p_t \in \Delta^{|\mathcal V|-1}\right)\right\}_{t=1}^{T}$ be any sequence of samples, and let
$\left\{\eta_t \in (0,1)\right\}_{t=1}^{T}$ be any sequence of step-sizes. Assume the parameters are randomly initialized and updated by gradient descent:
\begin{equation*}
\theta_0 \sim \mu,
\qquad
\mu \ll \mathrm{Leb}_p,
\end{equation*}
\begin{equation*}
\theta_{t+1}
=
\theta_t
-
\eta_t
\nabla
\mathcal L_{s_t,p_t}(\theta_t),
\end{equation*}
where $\mathrm{Leb}_p$ denotes Lebesgue measure on $\mathbb R^p$ and $\mathcal L_{s,p}:\mathbb R^p \to \mathbb R$ is the standard cross-entropy loss:
\begin{equation*}
\mathcal L_{s,p}(\theta)
=
\mathrm{CrossEntropy}
\bigl(
f(s;\theta),
p
\bigr).
\end{equation*}
Then, with probability one over the draw of $\theta_0$, the last-token, last-layer representation map
\begin{equation*}
\mathcal V^{\le K}
\ni
s
\longmapsto
\mathbf r(s;\theta_T)
\in
\mathbb R^d
\end{equation*}
is injective. Equivalently,
\begin{equation*}
\Pr
\left[
\exists\,
s \neq t
\in
\mathcal V^{\le K}
:
\mathbf r(s;\theta_T)
=
\mathbf r(t;\theta_T)
\right]
=
0,
\end{equation*}
where $\mathbf r(\cdot;\theta_T)$ denotes the last-token, last-layer representation of the TLM after $T$ gradient descent steps.
\end{knownresult}

\paragraph{From the known result to VLA policies.}
We now explain how the above known result is used in our VLA setting. 
Let
\begin{equation*}
\mathcal X_{\mathrm{VL}}
\subseteq
\mathcal O \times \mathcal L
\end{equation*}
denote the finite set of visual-language inputs considered in training or evaluation, where \(o \in \mathcal O\) is a visual observation and \(\ell \in \mathcal L\) is a language instruction. We write each input as
\begin{equation*}
x=(o,\ell)\in \mathcal X_{\mathrm{VL}}.
\end{equation*}
A VLA policy can be decomposed into a VLM representation map followed by an action head:
\begin{equation*}
\pi_{\theta,\phi}^{\mathrm{VLA}}
=
G_{\phi}^{\mathrm{act}}
\circ
F_{\theta}^{\mathrm{VLM}},
\end{equation*}
where
\begin{equation*}
F_{\theta}^{\mathrm{VLM}}
:
\mathcal X_{\mathrm{VL}}
\longrightarrow
\mathbb R^d
\end{equation*}
maps a visual-language input to a VLM representation, and
\begin{equation*}
G_{\phi}^{\mathrm{act}}
:
\mathbb R^d
\longrightarrow
\mathcal Y_{\mathrm{act}}
\end{equation*}
maps the VLM representation to an action output, such as an action vector, action logits, or action-distribution parameters.

For a VLM whose multimodal tokens are processed by a causal decoder backbone, the visual-language input is first converted into a bounded multimodal token context. We denote this conversion by
\begin{equation*}
\tau:
\mathcal X_{\mathrm{VL}}
\longrightarrow
\mathcal V_{\mathrm{VL}}^{\le K_{\mathrm{VL}}},
\end{equation*}
where \(\mathcal V_{\mathrm{VL}}\) is a finite multimodal vocabulary containing visual tokens, language tokens, and special modality or separator tokens. The VLM representation can then be written as
\begin{equation*}
F_{\theta}^{\mathrm{VLM}}(x)
=
\mathbf r(\tau(x);\theta),
\qquad
x\in\mathcal X_{\mathrm{VL}},
\end{equation*}
where \(\mathbf r(\cdot;\theta)\) is the last-token representation map of the causal decoder.

The known result is originally stated for decoder-only language Transformers, but the mathematical object it controls is the finite-context representation map of a causal decoder:
\begin{equation*}
s
\longmapsto
\mathbf r(s;\theta_T),
\qquad
s\in \mathcal V^{\le K}.
\end{equation*}
Therefore, after replacing the language vocabulary and context length by the multimodal vocabulary and context length,
\begin{equation*}
\mathcal V
=
\mathcal V_{\mathrm{VL}},
\qquad
K
=
K_{\mathrm{VL}},
\end{equation*}
the same decoder-level conclusion gives
\begin{equation*}
\Pr_{\theta_0}
\left[
\exists\,
s\neq s'
\in
\mathcal V_{\mathrm{VL}}^{\le K_{\mathrm{VL}}}
:
\mathbf r(s;\theta_T)
=
\mathbf r(s';\theta_T)
\right]
=
0.
\end{equation*}

Assume that the multimodal conversion map \(\tau\) is injective on the considered visual-language input set:
\begin{equation*}
x\neq x'
\quad
\Longrightarrow
\quad
\tau(x)\neq \tau(x'),
\qquad
x,x'\in \mathcal X_{\mathrm{VL}}.
\end{equation*}
Then the VLM representation map is almost surely injective on \(\mathcal X_{\mathrm{VL}}\). Indeed, for any \(x\neq x'\), the injectivity of \(\tau\) gives
\begin{equation*}
\tau(x)\neq \tau(x'),
\end{equation*}
and the decoder-level injectivity gives
\begin{equation*}
\mathbf r(\tau(x);\theta_T)
\neq
\mathbf r(\tau(x');\theta_T)
\end{equation*}
with probability one. Hence,
\begin{equation*}
F_{\theta_T}^{\mathrm{VLM}}(x)
\neq
F_{\theta_T}^{\mathrm{VLM}}(x'),
\end{equation*}
and therefore
\begin{equation*}
\Pr_{\theta_0}
\left[
\exists\,
x\neq x'
\in
\mathcal X_{\mathrm{VL}}
:
F_{\theta_T}^{\mathrm{VLM}}(x)
=
F_{\theta_T}^{\mathrm{VLM}}(x')
\right]
=
0.
\end{equation*}
Thus, under the finite-context assumptions of the known result, the VLM representation map is injective on the considered visual-language inputs.

We next consider the action head. Define the finite set of VLM representations induced by \(\mathcal X_{\mathrm{VL}}\) as
\begin{equation*}
\mathcal H_{\theta_T}
=
F_{\theta_T}^{\mathrm{VLM}}
\left(
\mathcal X_{\mathrm{VL}}
\right)
=
\left\{
F_{\theta_T}^{\mathrm{VLM}}(x)
:
x\in \mathcal X_{\mathrm{VL}}
\right\}
\subseteq
\mathbb R^d.
\end{equation*}
We require the action head to be injective on this induced representation set:
\begin{equation*}
h\neq h',
\quad
h,h'\in \mathcal H_{\theta_T}
\quad
\Longrightarrow
\quad
G_{\phi}^{\mathrm{act}}(h)
\neq
G_{\phi}^{\mathrm{act}}(h').
\end{equation*}
This is a finite-set injectivity condition on the representations that actually arise from the VLM, rather than a global injectivity claim on all of \(\mathbb R^d\).

For standard non-degenerate affine or MLP action heads, this finite-set condition is generic. Since \(\mathcal H_{\theta_T}\) is finite, for any fixed pair \(h\neq h'\in \mathcal H_{\theta_T}\), the collision event
\begin{equation*}
G_{\phi}^{\mathrm{act}}(h)
=
G_{\phi}^{\mathrm{act}}(h')
\end{equation*}
defines a lower-dimensional subset of the action-head parameter space under continuous, non-degenerate parameterization. Taking a finite union over all pairs \(h\neq h'\) in \(\mathcal H_{\theta_T}\), the probability of an action-head collision is zero under absolutely continuous initialization of \(\phi\). Hence,
\begin{equation*}
\Pr_{\phi}
\left[
\exists\,
h\neq h'
\in
\mathcal H_{\theta_T}
:
G_{\phi}^{\mathrm{act}}(h)
=
G_{\phi}^{\mathrm{act}}(h')
\right]
=
0.
\end{equation*}

Combining the injectivity of the VLM representation map and the injectivity of the action head on the induced representation set, we obtain
\begin{equation*}
x\neq x'
\quad
\Longrightarrow
\quad
F_{\theta_T}^{\mathrm{VLM}}(x)
\neq
F_{\theta_T}^{\mathrm{VLM}}(x')
\quad
\Longrightarrow
\quad
G_{\phi}^{\mathrm{act}}
\left(
F_{\theta_T}^{\mathrm{VLM}}(x)
\right)
\neq
G_{\phi}^{\mathrm{act}}
\left(
F_{\theta_T}^{\mathrm{VLM}}(x')
\right).
\end{equation*}
Therefore, the full VLA mapping
\begin{equation*}
\pi_{\theta_T,\phi}^{\mathrm{VLA}}
=
G_{\phi}^{\mathrm{act}}
\circ
F_{\theta_T}^{\mathrm{VLM}}
\end{equation*}
is injective on the considered visual-language input set \(\mathcal X_{\mathrm{VL}}\):
\begin{equation*}
\Pr
\left[
\exists\,
x\neq x'
\in
\mathcal X_{\mathrm{VL}}
:
\pi_{\theta_T,\phi}^{\mathrm{VLA}}(x)
=
\pi_{\theta_T,\phi}^{\mathrm{VLA}}(x')
\right]
=
0.
\end{equation*}

In summary, the known result is used to justify the injectivity of the VLM representation map
\begin{equation*}
F_{\theta_T}^{\mathrm{VLM}}
:
x
\longmapsto
\mathbf r(\tau(x);\theta_T),
\end{equation*}
while the finite-set injectivity of the action head ensures that distinct VLM representations remain distinct after being mapped to action outputs. Consequently, on the finite set of visual-language inputs considered in our VLA setting, the composition
\begin{equation*}
G_{\phi}^{\mathrm{act}}
\circ
F_{\theta_T}^{\mathrm{VLM}}
\end{equation*}
is injective.

\subsection{Proof of the Minimal VLM--Action Quotient}
\label{app:proof_minimal_vlm_action_quotient}

\begin{proof}
We prove the statement on the support
\(\mathcal H_{\mathrm{sup}}=\operatorname{supp}(P_H)\). Fix the
Bayes-optimal action-trajectory law \(p_X^\star\) used in
Section~\ref{sec:minimal_vlm_action_quotient}. By assumption, this law depends
on the input \(x\) only through its VLM latent \(h_\theta(x)\), namely
\[
p_X^\star(\cdot\mid x)
=
p_H^\star(\cdot\mid h_\theta(x))
\]
for \(P_X\)-almost every \(x\).

First, observe that \(\sim_{\mathcal A}\) is an equivalence relation on
\(\mathcal H_{\mathrm{sup}}\). Indeed, it is defined by equality of the induced
Bayes-optimal action-trajectory laws:
\[
h\sim_{\mathcal A}h'
\quad\Longleftrightarrow\quad
p_H^\star(\cdot\mid h)=p_H^\star(\cdot\mid h').
\]
Hence reflexivity, symmetry, and transitivity follow directly from the
corresponding properties of equality.

Let
\[
q_{\mathcal A}^\star(h)=[h]_{\sim_{\mathcal A}},
\qquad
\mathcal H_{\mathcal A}^\star
=
\mathcal H_{\mathrm{sup}}/\sim_{\mathcal A}.
\]
We first show that \(q_{\mathcal A}^\star\) is exact action-sufficient. Define
\[
\bar p_{\mathcal A}^\star
:
\mathcal H_{\mathcal A}^\star
\to
\Delta(\mathcal A_{\mathrm{seq}})
\]
by
\[
\bar p_{\mathcal A}^\star\bigl([h]_{\sim_{\mathcal A}}\bigr)
:=
p_H^\star(\cdot\mid h).
\]
This definition is well-defined: if \(h\) and \(h'\) belong to the same
equivalence class, then \(h\sim_{\mathcal A}h'\), and therefore
\[
p_H^\star(\cdot\mid h)
=
p_H^\star(\cdot\mid h').
\]
Thus the value of \(\bar p_{\mathcal A}^\star\) does not depend on the chosen
representative of the equivalence class. Consequently, for \(P_X\)-almost every
\(x\),
\[
\bar p_{\mathcal A}^\star
\bigl(q_{\mathcal A}^\star(h_\theta(x))\bigr)
=
\bar p_{\mathcal A}^\star
\bigl([h_\theta(x)]_{\sim_{\mathcal A}}\bigr)
=
p_H^\star(\cdot\mid h_\theta(x))
=
p_X^\star(\cdot\mid x).
\]
Therefore \(q_{\mathcal A}^\star\) preserves exactly the information needed to
recover the Bayes-optimal action-trajectory law, and is exact
action-sufficient.

It remains to prove minimality. Let
\[
s:\mathcal H_{\mathrm{sup}}\to\mathcal S
\]
be any exact action-sufficient representation. By definition, there exists a
decoder
\[
\bar p_s:\mathcal S\to\Delta(\mathcal A_{\mathrm{seq}})
\]
such that, for every \(h\in\mathcal H_{\mathrm{sup}}\),
\[
p_H^\star(\cdot\mid h)=\bar p_s(s(h)).
\]
Consider any two latents \(h,h'\in\mathcal H_{\mathrm{sup}}\) such that
\(s(h)=s(h')\). Then exact action-sufficiency of \(s\) gives
\[
p_H^\star(\cdot\mid h)
=
\bar p_s(s(h))
=
\bar p_s(s(h'))
=
p_H^\star(\cdot\mid h').
\]
Hence \(h\sim_{\mathcal A}h'\). Therefore every fiber of \(s\) is contained in
one equivalence class of \(\sim_{\mathcal A}\). Equivalently, \(s\) may split
the equivalence classes of \(q_{\mathcal A}^\star\), but it cannot merge two
different equivalence classes.

This also gives an explicit factorization. Define
\[
\rho:s(\mathcal H_{\mathrm{sup}})\to\mathcal H_{\mathcal A}^\star
\]
by
\[
\rho(s(h)) := [h]_{\sim_{\mathcal A}}.
\]
The preceding argument shows that \(\rho\) is well-defined: if
\(s(h)=s(h')\), then \(h\sim_{\mathcal A}h'\), so
\([h]_{\sim_{\mathcal A}}=[h']_{\sim_{\mathcal A}}\). Hence, for all
\(h\in\mathcal H_{\mathrm{sup}}\),
\[
q_{\mathcal A}^\star(h)
=
\rho(s(h)).
\]
Thus \(q_{\mathcal A}^\star\) can be recovered from any exact
action-sufficient representation \(s\). In other words, any such \(s\) is at
least as fine as \(q_{\mathcal A}^\star\).

Finally, suppose a representation were strictly coarser than
\(q_{\mathcal A}^\star\) while remaining exact action-sufficient. Then it would
merge two latents \(h,h'\) belonging to different
\(\sim_{\mathcal A}\)-equivalence classes. But by definition of
\(\sim_{\mathcal A}\), this means
\[
p_H^\star(\cdot\mid h)
\neq
p_H^\star(\cdot\mid h'),
\]
whereas exact action-sufficiency would require the same decoded
Bayes-optimal law for any two latents mapped to the same representation value.
This is a contradiction. Therefore no strictly coarser representation can
remain exact action-sufficient.

Hence
\[
\mathcal H_{\mathcal A}^\star
=
\mathcal H_{\mathrm{sup}}/\sim_{\mathcal A}
\]
is the minimal action-sufficient quotient of the VLM representation space, and
this quotient is unique up to relabeling of equivalence classes.
\end{proof}

\subsection{Action-Sensitive Generalization}
\label{app:proof_action_sensitivity_generalization}

We prove Theorem~\ref{thm:action_quotient_generalization}. Let
\(S=\{(x_i,a^i_{1:T})\}_{i=1}^n\) be a training sample drawn from the
data distribution. For a predicted action trajectory
\(\hat a_{1:T}\), define the normalized trajectory norm
\[
\|\hat a_{1:T}\|_T
:=
\left(
\frac{1}{T}
\sum_{t=1}^T
\|\hat a_t\|_2^2
\right)^{1/2}.
\]
We assume that the action loss is \(L_\ell\)-Lipschitz with respect to this
norm, namely for any target trajectory \(a_{1:T}\),
\[
\bigl|
\ell_{\mathrm{act}}(\hat a_{1:T},a_{1:T})
-
\ell_{\mathrm{act}}(\hat a'_{1:T},a_{1:T})
\bigr|
\le
L_\ell
\|\hat a_{1:T}-\hat a'_{1:T}\|_T .
\]

Let \(\mathcal G_{\mathrm{act}}\) be the action-head class. For a
representation class \(\mathcal R\), define the induced action-trajectory
prediction class
\[
\mathcal F_{\mathcal R}
:=
\left\{
x
\mapsto
g(r(x))
:
r\in\mathcal R,\,
g\in\mathcal G_{\mathrm{act}}
\right\}.
\]
The corresponding loss class is
\[
\ell\circ\mathcal F_{\mathcal R}
:=
\left\{
(x,a_{1:T})
\mapsto
\ell_{\mathrm{act}}(f(x),a_{1:T})
:
f\in\mathcal F_{\mathcal R}
\right\}.
\]
We define the generalization gap of \(\mathcal R\) as
\[
\mathrm{Gen}(\mathcal R)
:=
\sup_{f\in\mathcal F_{\mathcal R}}
\left|
\mathbb E\bigl[
\ell_{\mathrm{act}}(f(X),A_{1:T})
\bigr]
-
\frac{1}{n}
\sum_{i=1}^n
\ell_{\mathrm{act}}(f(x_i),a^i_{1:T})
\right|.
\]

We now define the action-sensitive complexity. Let
\(\sigma_i=(\sigma_{i,1},\ldots,\sigma_{i,T})\) be independent Rademacher
vectors with the same shape as an action trajectory, and define the normalized
inner product
\[
\langle u_{1:T},v_{1:T}\rangle_T
:=
\frac{1}{T}
\sum_{t=1}^T
\langle u_t,v_t\rangle_2 .
\]
The empirical action-sensitive complexity of \(\mathcal R\) on the input
sample \(S_X=\{x_i\}_{i=1}^n\) is
\[
\widehat{\mathcal E}_{\mathrm{act}}(\mathcal R;S_X)
:=
\sqrt{T}\,
\mathbb E_{\sigma}
\left[
\sup_{f\in\mathcal F_{\mathcal R}}
\frac{1}{n}
\sum_{i=1}^n
\left\langle
\sigma_i,
f(x_i)
\right\rangle_T
\right].
\]
We write \(\mathcal E_{\mathrm{act}}(\mathcal R)\) for either its expectation
over \(S_X\), or any high-probability upper envelope of the empirical quantity.
This term measures the extent to which variation in the representation class
can be converted by the action head into variation of predicted action
trajectories.

\begin{proof}
By standard symmetrization and concentration for the loss class
\(\ell\circ\mathcal F_{\mathcal R}\), with the remaining finite-sample,
capacity, and distribution-shift terms collected into \(B_0(\mathcal R)\), we
obtain
\[
\mathrm{Gen}(\mathcal R)
\le
B_0(\mathcal R)
+
C\,
\mathfrak R_n(\ell\circ\mathcal F_{\mathcal R}),
\]
where \(C>0\) is a universal constant and
\(\mathfrak R_n(\ell\circ\mathcal F_{\mathcal R})\) denotes the Rademacher
complexity of the induced loss class.

Because \(\ell_{\mathrm{act}}\) is \(L_\ell\)-Lipschitz in the predicted
trajectory, the vector contraction inequality gives
\[
\mathfrak R_n(\ell\circ\mathcal F_{\mathcal R})
\le
C L_\ell\,
\mathfrak R_n(\mathcal F_{\mathcal R}),
\]
where the constant \(C\) absorbs universal factors from the contraction
inequality. By the definition of
\(\mathcal E_{\mathrm{act}}(\mathcal R)\),
\[
\mathfrak R_n(\mathcal F_{\mathcal R})
=
\frac{1}{\sqrt T}
\mathcal E_{\mathrm{act}}(\mathcal R).
\]
Therefore,
\[
\mathrm{Gen}(\mathcal R)
\le
B_0(\mathcal R)
+
\frac{C L_\ell}{\sqrt T}
\mathcal E_{\mathrm{act}}(\mathcal R).
\]
This proves the first part of the theorem.

We now compare the raw VLM representation with the action quotient. Let
\[
H=h_\theta(X),
\qquad
Q=q_{\mathcal A}^\star(H).
\]
Let \(\mathcal R_H\) be the representation class induced by \(H\), and let
\(\mathcal R_Q\) be the representation class induced by \(Q\). By
Theorem~\ref{thm:minimal_vlm_action_quotient}, the quotient is exact
action-sufficient: there exists
\[
\bar p_{\mathcal A}^\star:
\mathcal H_{\mathcal A}^\star
\to
\Delta(\mathcal A_{\mathrm{seq}})
\]
such that
\[
p_X^\star(\cdot\mid x)
=
\bar p_{\mathcal A}^\star
\bigl(
q_{\mathcal A}^\star(h_\theta(x))
\bigr).
\]
Hence replacing \(H\) by \(Q\) does not remove any information needed to
recover the Bayes-optimal action-trajectory law.

Next, we show that the quotient does not increase the action-sensitive
complexity. Consider any predictor in the quotient-induced action class. It
has the form
\[
x
\mapsto
g\bigl(q_{\mathcal A}^\star(h_\theta(x))\bigr)
\]
for some action head \(g\). Define a raw-latent action head
\[
\tilde g(h)
:=
g\bigl(q_{\mathcal A}^\star(h)\bigr).
\]
Under the regularity condition that the raw action-head class is closed under
composition with the quotient map, \(\tilde g\) is an admissible action head on
the raw VLM latent space. Therefore every quotient-based predictor can be
realized as a raw-latent predictor:
\[
\mathcal F_{\mathcal R_Q}
\subseteq
\mathcal F_{\mathcal R_H}.
\]
Since Rademacher complexity is monotone under function-class inclusion,
\[
\mathfrak R_n(\mathcal F_{\mathcal R_Q})
\le
\mathfrak R_n(\mathcal F_{\mathcal R_H}).
\]
Equivalently,
\[
\mathcal E_{\mathrm{act}}(\mathcal R_Q)
\le
\mathcal E_{\mathrm{act}}(\mathcal R_H).
\]

The remaining term \(B_0(\mathcal R)\) contains the parts of the bound that
are not captured by the action-sensitive trajectory complexity, such as
finite-sample concentration, residual capacity terms, and distribution-shift
terms. These terms are assumed to be monotone under restriction of the induced
action-prediction class. Since
\(\mathcal F_{\mathcal R_Q}\subseteq\mathcal F_{\mathcal R_H}\), and since the
exact action-sufficiency of \(q_{\mathcal A}^\star\) prevents the quotient from
introducing additional Bayes-action approximation error, we have
\[
B_0(\mathcal R_Q)
\le
B_0(\mathcal R_H).
\]
Combining this inequality with the monotonicity of
\(\mathcal E_{\mathrm{act}}\), we obtain
\[
\begin{aligned}
\mathrm{Bd}(\mathcal R_Q)
&=
B_0(\mathcal R_Q)
+
\frac{C L_\ell}{\sqrt T}
\mathcal E_{\mathrm{act}}(\mathcal R_Q)
\\
&\le
B_0(\mathcal R_H)
+
\frac{C L_\ell}{\sqrt T}
\mathcal E_{\mathrm{act}}(\mathcal R_H)
\\
&=
\mathrm{Bd}(\mathcal R_H).
\end{aligned}
\]
Applying the first part of the theorem to \(\mathcal R_Q\) gives
\[
\mathrm{Gen}(\mathcal R_Q)
\le
\mathrm{Bd}(\mathcal R_Q)
\le
\mathrm{Bd}(\mathcal R_H).
\]

Finally, suppose the quotient removes a nontrivial action-irrelevant direction:
there exist latents \(h,h'\in\mathcal H_{\mathrm{sup}}\) such that
\[
h\neq h',
\qquad
q_{\mathcal A}^\star(h)=q_{\mathcal A}^\star(h'),
\qquad
p_H^\star(\cdot\mid h)=p_H^\star(\cdot\mid h').
\]
These two latents require the same Bayes-optimal action behavior, but a raw
action head may still distinguish them and convert their difference into
different predicted trajectories. The quotient forces all predictors in
\(\mathcal F_{\mathcal R_Q}\) to be invariant along such directions. Hence, if
the removed directions contribute positive action-sensitive complexity for the
raw class, then
\[
\mathcal E_{\mathrm{act}}(\mathcal R_Q)
<
\mathcal E_{\mathrm{act}}(\mathcal R_H).
\]
Together with
\(B_0(\mathcal R_Q)\le B_0(\mathcal R_H)\), this yields a strictly smaller
certified bound. Thus the action quotient gives a no-worse certified
generalization bound, with strict improvement whenever it removes nontrivial
action-irrelevant variation that the raw action head could otherwise amplify.
\end{proof}

\subsection{robotwin2.0}

The RoboTwin 2.0 results demonstrate the robustness advantage of QuoVLA on the nominal Hard split. Although QuoVLA does not achieve the highest Easy-setting average, it attains the best Hard-setting average of 58.6\%, improving over \(\pi_{0.5}\) by 14.76 percentage points and substantially outperforming ACT, \(\pi_0\), and DP3. This advantage is especially pronounced in tasks such as \textit{Scan Object}, \textit{Stack Blocks Two}, \textit{Stack Blocks Three}, \textit{Stamp Seal}, and \textit{Turn Switch}, where QuoVLA maintains much higher success rates under the Hard setting. These results suggest that the quotient-based representation effectively suppresses action-irrelevant variations while preserving stable action generation. The apparent Easy--Hard reversal is not necessarily contradictory, since the RoboTwin 2.0 Easy and Hard splits may differ in instance composition and perturbation profiles rather than forming a strictly nested difficulty hierarchy; under such variations, QuoVLA's quotient bottleneck can better filter nuisance factors that are irrelevant to the required action behavior.

\begin{table}[t]
\centering
\caption{\textbf{Evaluation on RoboTwin 2.0 Simulation.}}
\label{tab:robotwin2_sim}
\scriptsize
\setlength{\tabcolsep}{3pt}
\renewcommand{\arraystretch}{0.92}
\resizebox{0.9\linewidth}{!}{%
\begin{tabular}{@{}lcccccccccc@{}}
\toprule
\textbf{Simulation Task}
& \multicolumn{2}{c}{\textbf{ACT}}
& \multicolumn{2}{c}{$\boldsymbol{\pi}_{0}$}
& \multicolumn{2}{c}{$\boldsymbol{\pi}_{0.5}$}
& \multicolumn{2}{c}{\textbf{DP3}}
& \multicolumn{2}{c}{\textbf{QuoVLA}} \\
\cmidrule(lr){2-3}
\cmidrule(lr){4-5}
\cmidrule(lr){6-7}
\cmidrule(lr){8-9}
\cmidrule(lr){10-11}
& \textbf{Easy} & \textbf{Hard}
& \textbf{Easy} & \textbf{Hard}
& \textbf{Easy} & \textbf{Hard}
& \textbf{Easy} & \textbf{Hard}
& \textbf{Easy} & \textbf{Hard} \\
\midrule
\textit{Adjust Bottle} & 97\% & 23\% & 90\% & 56\% & 79\% & 83\% & 99\% & 3\% & 18\% & 68\% \\
\textit{Beat Block Hammer} & 56\% & 3\% & 43\% & 21\% & 63\% & 50\% & 72\% & 8\% & 20\% & 42\% \\
\textit{Blocks Ranking RGB} & 1\% & 0\% & 19\% & 5\% & 43\% & 35\% & 3\% & 0\% & 50\% & 62\% \\
\textit{Blocks Ranking Size} & 0\% & 0\% & 7\% & 1\% & 8\% & 14\% & 2\% & 0\% & 43\% & 45\% \\
\textit{Click Alarmclock} & 32\% & 4\% & 63\% & 11\% & 97\% & 93\% & 77\% & 14\% & 75\% & 72\% \\
\textit{Click Bell} & 58\% & 3\% & 44\% & 3\% & 75\% & 76\% & 90\% & 0\% & 27\% & 48\% \\
\textit{Dump Bin Bigbin} & 68\% & 1\% & 83\% & 24\% & 30\% & 42\% & 85\% & 53\% & 35\% & 60\% \\
\textit{Grab Roller} & 94\% & 25\% & 96\% & 80\% & 90\% & 89\% & 98\% & 2\% & 14\% & 86\% \\
\textit{Handover Block} & 42\% & 0\% & 45\% & 8\% & 18\% & 19\% & 70\% & 0\% & 45\% & 50\% \\
\textit{Handover Mic} & 85\% & 0\% & 98\% & 13\% & 28\% & 18\% & 100\% & 3\% & 36\% & 37\% \\
\textit{Hanging Mug} & 7\% & 0\% & 11\% & 3\% & 3\% & 3\% & 17\% & 1\% & 11\% & 22\% \\
\textit{Lift Pot} & 88\% & 0\% & 84\% & 36\% & 0\% & 0\% & 97\% & 0\% & 7\% & 19\% \\
\textit{Move Can Pot} & 22\% & 4\% & 58\% & 21\% & 29\% & 27\% & 70\% & 6\% & 36\% & 46\% \\
\textit{Move Pillbottle Pad} & 0\% & 0\% & 21\% & 1\% & 33\% & 29\% & 41\% & 0\% & 36\% & 45\% \\
\textit{Move Playingcard Away} & 36\% & 0\% & 53\% & 22\% & 59\% & 67\% & 68\% & 3\% & 66\% & 86\% \\
\textit{Move Stapler Pad} & 0\% & 0\% & 0\% & 2\% & 16\% & 18\% & 12\% & 0\% & 23\% & 37\% \\
\textit{Open Laptop} & 56\% & 0\% & 85\% & 46\% & 19\% & 35\% & 82\% & 7\% & 26\% & 54\% \\
\textit{Open Microwave} & 86\% & 0\% & 80\% & 50\% & 35\% & 37\% & 61\% & 22\% & 42\% & 56\% \\
\textit{Pick Diverse Bottles} & 7\% & 0\% & 27\% & 6\% & 5\% & 3\% & 52\% & 1\% & 12\% & 22\% \\
\textit{Pick Dual Bottles} & 31\% & 0\% & 57\% & 12\% & 10\% & 6\% & 60\% & 1\% & 17\% & 25\% \\
\textit{Place A2b Left} & 1\% & 0\% & 31\% & 1\% & 62\% & 60\% & 46\% & 2\% & 69\% & 79\% \\
\textit{Place A2b Right} & 0\% & 0\% & 27\% & 6\% & 62\% & 57\% & 49\% & 0\% & 69\% & 76\% \\
\textit{Place Bread Basket} & 6\% & 0\% & 17\% & 4\% & 48\% & 56\% & 26\% & 1\% & 55\% & 75\% \\
\textit{Place Bread Skillet} & 7\% & 0\% & 23\% & 1\% & 38\% & 46\% & 19\% & 0\% & 45\% & 65\% \\
\textit{Place Burger Fries} & 49\% & 0\% & 80\% & 4\% & 66\% & 70\% & 72\% & 18\% & 73\% & 89\% \\
\textit{Place Can Basket} & 1\% & 0\% & 41\% & 5\% & 19\% & 25\% & 67\% & 2\% & 26\% & 44\% \\
\textit{Place Cans Plasticbox} & 16\% & 0\% & 34\% & 2\% & 40\% & 47\% & 48\% & 3\% & 47\% & 66\% \\
\textit{Place Container Plate} & 72\% & 1\% & 88\% & 45\% & 71\% & 78\% & 86\% & 1\% & 78\% & 97\% \\
\textit{Place Dual Shoes} & 9\% & 0\% & 15\% & 0\% & 12\% & 7\% & 13\% & 0\% & 19\% & 26\% \\
\textit{Place Empty Cup} & 61\% & 0\% & 37\% & 11\% & 75\% & 86\% & 65\% & 1\% & 82\% & 100\% \\
\textit{Place Fan} & 1\% & 0\% & 20\% & 10\% & 25\% & 36\% & 36\% & 1\% & 32\% & 55\% \\
\textit{Place Mouse Pad} & 0\% & 0\% & 7\% & 1\% & 21\% & 26\% & 4\% & 1\% & 28\% & 45\% \\
\textit{Place Object Basket} & 15\% & 0\% & 16\% & 2\% & 43\% & 36\% & 65\% & 0\% & 50\% & 55\% \\
\textit{Place Object Scale} & 0\% & 0\% & 10\% & 0\% & 40\% & 49\% & 15\% & 0\% & 47\% & 68\% \\
\textit{Place Object Stand} & 1\% & 0\% & 36\% & 11\% & 74\% & 65\% & 60\% & 0\% & 81\% & 83\% \\
\textit{Place Phone Stand} & 2\% & 0\% & 35\% & 7\% & 49\% & 53\% & 44\% & 2\% & 56\% & 70\% \\
\textit{Place Shoe} & 5\% & 0\% & 28\% & 6\% & 57\% & 61\% & 58\% & 2\% & 64\% & 78\% \\
\textit{Press Stapler} & 31\% & 6\% & 62\% & 29\% & 80\% & 70\% & 69\% & 3\% & 87\% & 87\% \\
\textit{Put Bottles Dustbin} & 27\% & 1\% & 54\% & 13\% & 12\% & 9\% & 60\% & 21\% & 19\% & 26\% \\
\textit{Put Object Cabinet} & 15\% & 0\% & 68\% & 18\% & 24\% & 15\% & 72\% & 1\% & 31\% & 32\% \\
\textit{Rotate Qrcode} & 1\% & 0\% & 68\% & 15\% & 47\% & 56\% & 74\% & 1\% & 54\% & 73\% \\
\textit{Scan Object} & 2\% & 0\% & 18\% & 1\% & 42\% & 38\% & 31\% & 1\% & 49\% & 55\% \\
\textit{Shake Bottle Horizontally} & 63\% & 4\% & 99\% & 51\% & 96\% & 100\% & 100\% & 25\% & 100\% & 100\% \\
\textit{Shake Bottle} & 74\% & 10\% & 97\% & 60\% & 91\% & 100\% & 98\% & 19\% & 98\% & 100\% \\
\textit{Stack Blocks Three} & 0\% & 0\% & 17\% & 0\% & 15\% & 16\% & 1\% & 0\% & 22\% & 33\% \\
\textit{Stack Blocks Two} & 25\% & 0\% & 42\% & 1\% & 48\% & 56\% & 24\% & 0\% & 55\% & 73\% \\
\textit{Stack Bowls Three} & 48\% & 0\% & 66\% & 24\% & 33\% & 35\% & 57\% & 5\% & 40\% & 52\% \\
\textit{Stack Bowls Two} & 82\% & 0\% & 91\% & 41\% & 78\% & 66\% & 83\% & 6\% & 61\% & 61\% \\
\textit{Stamp Seal} & 2\% & 0\% & 3\% & 4\% & 36\% & 23\% & 18\% & 0\% & 43\% & 40\% \\
\textit{Turn Switch} & 5\% & 2\% & 27\% & 23\% & 5\% & 6\% & 46\% & 8\% & 36\% & 45\% \\
\midrule
\textbf{\textit{Average (in \%)}} 
& 29.7 & 1.7
& 46.4 & 16.3
& 42.98 & 43.84
& 55.2 & 5.0
& 45.1 & 58.6 \\
\bottomrule
\end{tabular}%
}
\end{table}

\subsection{Sensitivity to Temporal-Complexity Weight}
\label{sec:lambda_tc_sensitivity}

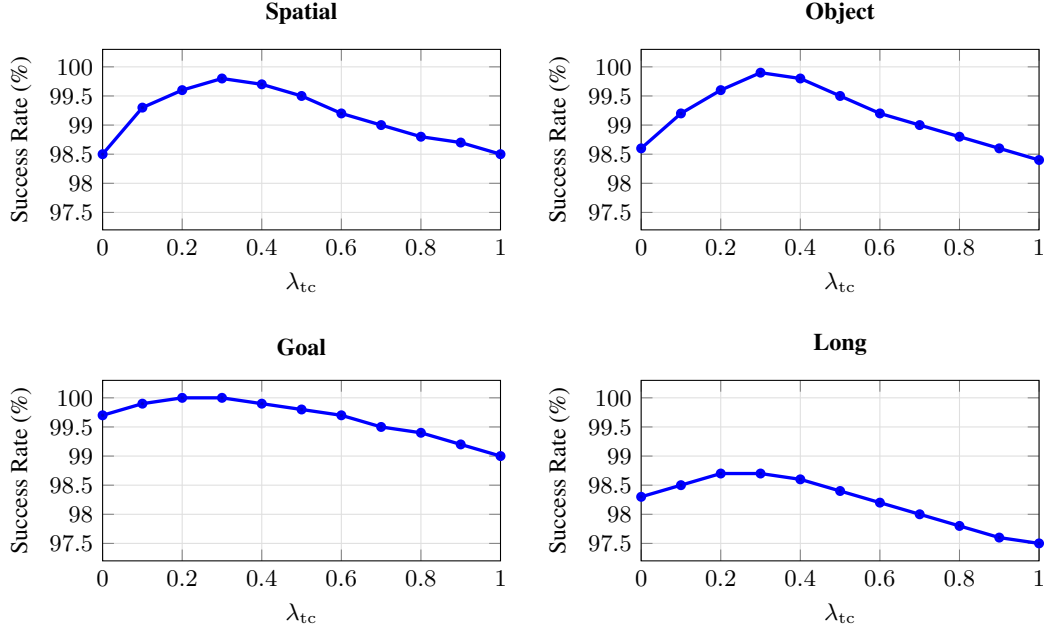
\begin{figure}[t]
\centering

\begin{subfigure}[t]{0.49\linewidth}
\centering
\begin{tikzpicture}
\begin{axis}[
    width=\linewidth,
    height=0.58\linewidth,
    title={\textbf{Spatial}},
    xlabel={$\lambda_{\mathrm{tc}}$},
    ylabel={Success Rate (\%)},
    xmin=0, xmax=1.0,
    ymin=97.2, ymax=100.3,
    xtick={0,0.2,0.4,0.6,0.8,1.0},
    ytick={97.5,98.0,98.5,99.0,99.5,100.0},
    grid=major,
    major grid style={draw=gray!25},
    tick label style={font=\small},
    label style={font=\small},
    title style={font=\small},
    every axis plot/.append style={
        line width=1.25pt,
        mark size=1.25pt
    }
]
\addplot+[mark=*] coordinates {
    (0.0,98.5)
    (0.1,99.3)
    (0.2,99.6)
    (0.3,99.8)
    (0.4,99.7)
    (0.5,99.5)
    (0.6,99.2)
    (0.7,99.0)
    (0.8,98.8)
    (0.9,98.7)
    (1.0,98.5)
};
\end{axis}
\end{tikzpicture}
\end{subfigure}
\hfill
\begin{subfigure}[t]{0.49\linewidth}
\centering
\begin{tikzpicture}
\begin{axis}[
    width=\linewidth,
    height=0.58\linewidth,
    title={\textbf{Object}},
    xlabel={$\lambda_{\mathrm{tc}}$},
    ylabel={Success Rate (\%)},
    xmin=0, xmax=1.0,
    ymin=97.2, ymax=100.3,
    xtick={0,0.2,0.4,0.6,0.8,1.0},
    ytick={97.5,98.0,98.5,99.0,99.5,100.0},
    grid=major,
    major grid style={draw=gray!25},
    tick label style={font=\small},
    label style={font=\small},
    title style={font=\small},
    every axis plot/.append style={
        line width=1.25pt,
        mark size=1.25pt
    }
]
\addplot+[mark=*] coordinates {
    (0.0,98.6)
    (0.1,99.2)
    (0.2,99.6)
    (0.3,99.9)
    (0.4,99.8)
    (0.5,99.5)
    (0.6,99.2)
    (0.7,99.0)
    (0.8,98.8)
    (0.9,98.6)
    (1.0,98.4)
};
\end{axis}
\end{tikzpicture}
\end{subfigure}

\vspace{0.8em}

\begin{subfigure}[t]{0.49\linewidth}
\centering
\begin{tikzpicture}
\begin{axis}[
    width=\linewidth,
    height=0.58\linewidth,
    title={\textbf{Goal}},
    xlabel={$\lambda_{\mathrm{tc}}$},
    ylabel={Success Rate (\%)},
    xmin=0, xmax=1.0,
    ymin=97.2, ymax=100.3,
    xtick={0,0.2,0.4,0.6,0.8,1.0},
    ytick={97.5,98.0,98.5,99.0,99.5,100.0},
    grid=major,
    major grid style={draw=gray!25},
    tick label style={font=\small},
    label style={font=\small},
    title style={font=\small},
    every axis plot/.append style={
        line width=1.25pt,
        mark size=1.25pt
    }
]
\addplot+[mark=*] coordinates {
    (0.0,99.7)
    (0.1,99.9)
    (0.2,100.0)
    (0.3,100.0)
    (0.4,99.9)
    (0.5,99.8)
    (0.6,99.7)
    (0.7,99.5)
    (0.8,99.4)
    (0.9,99.2)
    (1.0,99.0)
};
\end{axis}
\end{tikzpicture}
\end{subfigure}
\hfill
\begin{subfigure}[t]{0.49\linewidth}
\centering
\begin{tikzpicture}
\begin{axis}[
    width=\linewidth,
    height=0.58\linewidth,
    title={\textbf{Long}},
    xlabel={$\lambda_{\mathrm{tc}}$},
    ylabel={Success Rate (\%)},
    xmin=0, xmax=1.0,
    ymin=97.2, ymax=100.3,
    xtick={0,0.2,0.4,0.6,0.8,1.0},
    ytick={97.5,98.0,98.5,99.0,99.5,100.0},
    grid=major,
    major grid style={draw=gray!25},
    tick label style={font=\small},
    label style={font=\small},
    title style={font=\small},
    every axis plot/.append style={
        line width=1.25pt,
        mark size=1.25pt
    }
]
\addplot+[mark=*] coordinates {
    (0.0,98.3)
    (0.1,98.5)
    (0.2,98.7)
    (0.3,98.7)
    (0.4,98.6)
    (0.5,98.4)
    (0.6,98.2)
    (0.7,98.0)
    (0.8,97.8)
    (0.9,97.6)
    (1.0,97.5)
};
\end{axis}
\end{tikzpicture}
\end{subfigure}

\caption{
Sensitivity analysis of the temporal-complexity weight $\lambda_{\mathrm{tc}}$ on the four LIBERO suites.
The success rate remains stable across a broad range of $\lambda_{\mathrm{tc}}$, while overly large values cause a mild performance drop.
}
\label{fig:lambda_tc_sensitivity}
\end{figure}

\paragraph{Sensitivity analysis of $\lambda_{\mathrm{tc}}$.}
Figure~\ref{fig:lambda_tc_sensitivity} shows the sensitivity of QuoVLA to the temporal-complexity weight $\lambda_{\mathrm{tc}}$ on the four LIBERO suites. When $\lambda_{\mathrm{tc}}=0$, the temporal-complexity constraint is removed, and the success rates are 98.5\%, 98.6\%, 99.7\%, and 98.3\% on Spatial, Object, Goal, and Long, respectively. As $\lambda_{\mathrm{tc}}$ increases, the performance remains largely stable across all suites, indicating that QuoVLA is not sensitive to the exact choice of this coefficient. A moderate value slightly improves or maintains the success rate by suppressing unnecessary temporal fluctuations, while an overly large value introduces mild over-regularization and leads to a small overall decline. These results suggest that the relative temporal-complexity constraint is robust to hyperparameter selection and mainly acts as a stabilizing regularizer rather than a fragile tuning component.

\section{Broader Impacts}

This work aims to improve the robustness and data efficiency of vision-language-action models for robot control. Its potential positive impact lies in reducing the amount of task-specific robot data and engineering effort needed to adapt pretrained vision-language representations to manipulation tasks, which may support safer and more reliable robot learning in controlled environments. At the same time, more capable robot policies may introduce negative societal impacts if they are deployed in open-world or safety-critical settings without sufficient validation. In particular, failures under distribution shift, unintended physical interactions, or misuse of general robot-control systems could lead to safety risks. Our experiments are limited to public simulation benchmarks and controlled tabletop manipulation tasks, and we do not involve human subjects or private data. Practical deployment should therefore require additional safety checks, human oversight, and environment-specific risk assessment.



\end{document}